\documentclass[10pt,journal]{IEEEtran}

\ifCLASSOPTIONcompsoc
  \usepackage[nocompress]{cite}
\else
  \usepackage{cite}
\fi

\ifCLASSINFOpdf
\else
\fi

\usepackage{subfigure}
\usepackage{graphicx}

\usepackage{amsmath,amsthm}
\usepackage{amssymb,wasysym}
\usepackage{amsfonts,balance}
\usepackage{mathrsfs}
\usepackage{cite}
\usepackage{color}
\usepackage{url}

\usepackage{wrapfig}
\usepackage{verbatim}
\usepackage{dsfont}
\usepackage{amsmath}
\usepackage{algorithm}
 \usepackage{algpseudocode}
 \usepackage[dvipsnames]{xcolor}
\usepackage{xurl}
\usepackage{tabularx}
\usepackage[colorlinks=true,citecolor=black,linkcolor=black, urlcolor=black]{hyperref}

\theoremstyle{definition}
\newtheorem{theorem}{\textbf{Theorem}}

\theoremstyle{definition}
\newtheorem{lemma}{\textbf{Lemma}}

\theoremstyle{definition}

\theoremstyle{definition}
\newtheorem{definition}{\textbf{Definition}}

\theoremstyle{definition}
\newtheorem{assumption}{\textbf{Assumption}}

\theoremstyle{definition}
\newtheorem{remark}{\textbf{Remark}}
\allowdisplaybreaks[3]

\ifodd 0
\definecolor{darkgreen}{RGB}{0,200,0}
\else

\fi

\graphicspath{{figures/}}

\begin{document}

\title{FedMeld: A Model-dispersal Federated Learning Framework for Space-ground Integrated Networks}

\author{Qian Chen, \IEEEmembership{Member, IEEE}, Xianhao Chen, \IEEEmembership{Member, IEEE}, and Kaibin Huang, \IEEEmembership{Fellow,~IEEE}
\thanks{
Q. Chen, X. Chen, and K. Huang are with the Department of Electrical and
Electronic Engineering, The University of Hong Kong, Hong Kong (Email: \{qchen, xchen, huangkb\}@eee.hku.hk). 

\textit{(Corresponding author: Xianhao Chen and Kaibin Huang)}}
}

\markboth{Journal of \LaTeX\ Class Files}%
{Shell \MakeLowercase{\textit{et al.}}: Bare Advanced Demo of IEEEtran.cls for IEEE Computer Society Journals}


\IEEEtitleabstractindextext{
\begin{abstract}
To bridge the digital divide, space-ground integrated networks (SGINs) are expected to deliver artificial intelligence (AI) services to every corner of the world. One key mission of SGINs is to support federated learning (FL) at a global scale. However, existing space-ground integrated FL frameworks involve ground stations or costly inter-satellite links, entailing excessive training latency and communication costs. To overcome these limitations, we propose an infrastructure-free \underline{fed}erated learning framework based on a \underline{m}od\underline{el} \underline{d}ispersal (FedMeld) strategy, which exploits periodic movement patterns and store-carry-forward capabilities of satellites to enable parameter mixing across large-scale geographical regions. We theoretically show that FedMeld leads to global model convergence and quantify the effects of round interval and mixing ratio between adjacent areas on its learning performance. Based on the theoretical results, we formulate a joint optimization problem to design the staleness control and mixing ratio (SC-MR) for minimizing the training loss. By decomposing the problem into sequential SC and MR subproblems without compromising the optimality, we derive the round interval solution in a closed form and the mixing ratio in a semi-closed form to achieve the \textit{optimal} latency-accuracy tradeoff. Experiments using various datasets demonstrate that FedMeld achieves superior model accuracy while significantly reducing communication costs as compared with traditional FL schemes for SGINs.
\end{abstract}

}

\maketitle

\IEEEdisplaynontitleabstractindextext

\IEEEpeerreviewmaketitle

\section{Introduction}
Edge intelligence and space-ground integrated networks (SGINs) are two new features of sixth-generation (6G) mobile networks. Edge intelligence aims to deliver pervasive, real-time artificial intelligence (AI) services to users~\cite{8808168,qumobile}.
Satellites are particularly valuable in extending coverage to remote and underserved regions without terrestrial infrastructure, and in serving as a complementary tier to offload traffic in congested urban areas. With their wide coverage footprints enabling successive service across different regions, SGINs can ensure ubiquitous worldwide coverage~\cite{sheng2023coverage}.
The convergence of these two trends in 6G marks a transformative step toward integrating networking and intelligence in space, enabling AI services to reach every corner of the planet. For instance, the Starlink constellation has deployed over 30,000 Linux nodes and more than 6,000 microcontrollers in space~\cite{Space-Linux}, creating a satellite server network orbiting around the Earth. This advancement extends traditional 2D edge AI into a 3D space-ground integrated AI paradigm~\cite{chen2024space}, offering mission-critical AI services anytime, anywhere~\cite{10579820}.

One primary objective of space-ground integrated AI is to enable federated learning (FL)~\cite{chen2024space}, referred to as space-ground integrated FL (SGI-FL). Based on such technologies, pervasive ground clients collaboratively train AI models without sharing raw data~\cite{mcmahan2017communication,10453386,9562559}, while satellites act as parameter servers to periodically receive from, aggregate over, and redistribute model parameters to ground clients. 
Recent advances in direct satellite-to-device communication have further strengthened the practicality of this paradigm. With the inclusion of non-terrestrial networks in 3GPP~\cite{3GPP38.811,3GPP38.108} and the commercial adoption of satellite connectivity in mainstream smartphones (e.g., Apple iPhone and Huawei Mate 60), direct user–satellite links have become both technically feasible and cost-effective~\cite{wang2025constellation}. 
A representative use case of SGI-FL is healthcare: hospitals and clinics in sparsely populated areas can jointly train diagnostic or triage models without transmitting sensitive medical records, thereby protecting patient privacy and meeting strict data sovereignty requirements \cite{gdpr,10555359}. Similarly, in agriculture and climate monitoring, satellites and dispersed ground sensors can train global predictive models for crop diseases, drought risks, or extreme weather events, where cross-border regulations and site-specific sensitivities make centralized data collection impractical. In both scenarios, SGI-FL enables global-scale model convergence while safeguarding data privacy.

SGI-FL generally aims to train a global consensus model across a \textit{large-scale geographical area} by learning from distributed data. Since a parameter server in SGI-FL, such as a low Earth orbit (LEO) satellite, can only cover a limited area~\cite{9568716}, cooperation among multiple satellites is often needed.
Therefore, several classical FL algorithms, such as horizontal FL frameworks~\cite{10024766}, are unsuitable for the global training scenarios considered in SGI-FL. These frameworks rely on centralized cloud aggregation, where globally distributed clients are expected to upload model updates directly to a cloud server. Such a design requires stable and persistent connectivity, which is often impractical in SGI-FL schemes~\cite{sheng2023coverage,10579820}.
There are two primary approaches to synchronizing models among these satellite servers.
The first approach is \textit{mixing via ground stations}, where each satellite server transmits its aggregated model to a ground station for reaching global consensus. However, since a training round cannot finish until \textit{all} satellites transmit their local updates to the ground station and receive the updated global model~\cite{10121575, lin2023fedsn}, this approach often results in prolonged idle periods of up to several hours per round. 
The other approach is \textit{mixing via inter-satellite links} (ISLs), where satellites exchange the aggregated parameters with their neighboring satellites through ISLs. This approach can overcome the limited contact time between satellites and ground stations, accelerating the learning process through frequent model exchanges~\cite{10436088,9982621}. However, it generates substantial data traffic within the satellite constellations. Furthermore, ISLs are not fully supported in many LEO mega-constellations due to the prohibitively high cost of laser terminals~\cite{laserISL}. As a result, many LEO mega-constellations today still rely on bent-pipe frameworks, where data is routed via ground stations~\cite{frederiksen2024end,10623014}.
In a nutshell, existing SGI-FL schemes, when serving large-scale regions through multiple satellites, face three major limitations: 1) global model consensus often results in \textbf{slow convergence}; 2) \textbf{significant communication costs} occur over ISLs in large-scale constellations; and 3) satellite cooperation must rely on \textbf{extra infrastructure} such as ISLs, which are unavailable in many commercial satellites.

To overcome all the above issues, we propose an SGI-FL framework called \underline{Fed}erated learning with \underline{M}od\underline{el} \underline{D}ispersal (FedMeld) as our answer. Our mechanism is inspired by a key observation: Given the repetitive trajectories and store-carry-forward (SCF) capabilities of co-orbital satellites, as shown in Fig. \ref{fig:earth_mixing}, a satellite can transfer model parameters from one serving region $i$ (e.g., A) to the next region $i+1$ (e.g., B) by mixing models from these two adjacent areas; Similarly, the circular nature of satellite orbits ensures that the model parameters from region $i+1$ (e.g., B) will eventually be returned to region $i$ (e.g., A). Consequently, without requiring dedicated ground stations or ISLs, \textit{parameter mixing across all regions is naturally achieved through model dispersal via satellite trajectories}, akin to how animals disperse seeds across continents. Within the FedMeld framework, users can continue training while satellites carrying model parameters traverse between adjacent regions, introducing a latency-accuracy tradeoff in managing model staleness. Moreover, model mixing ratios across regions should also be determined judiciously to achieve optimal learning performance. 
The core advantages of FedMeld are summarized as follows:
\begin{itemize}
    \item \textbf{Infrastructure-free model aggregation:} 
 By utilizing repetitive trajectories and SCF behaviour of co-orbital satellites, FedMeld eliminates the reliance on ground stations or ISLs for model dispersal across regions, turning the satellite's orbital mobility from a foe to a friend.

    \item \textbf{Flexible control of latency-accuracy tradeoff:} The design introduces tunable parameters for staleness and mixing ratio, enabling clients to keep training without waiting for satellite movement and clients training in other regions.
\end{itemize}

\begin{figure}[t]
	\centering
    \includegraphics[width = 0.35\textwidth]{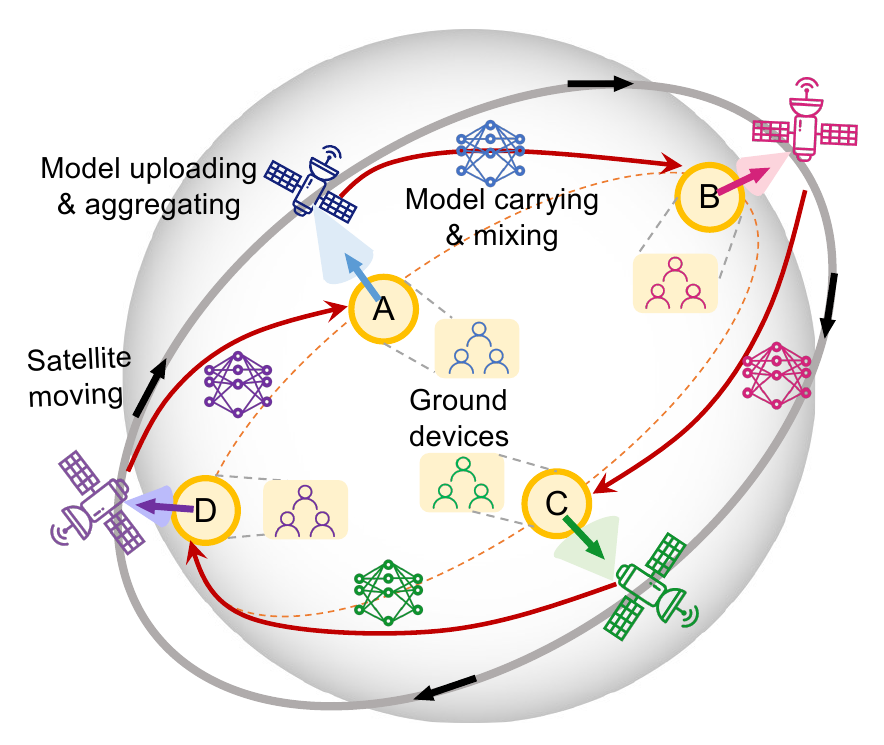}
\caption{Illustration of parameter mixing across different regions in the proposed FedMeld framework. \label{fig:earth_mixing}}
\end{figure}

To fully explore the advantages of FedMeld, we will answer two fundamental questions in this paper: 1) Does the proposed model dispersal method ensure global model convergence? and 2) how can we manage model staleness and mixing ratios to optimize training convergence? The key challenges in the convergence analysis of FedMeld primarily lie in the following two aspects.
\textbf{First}, while mobility-aware FL schemes have been explored for terrestrial networks~\cite{9154300,9759241,10128968,chen2024mobility}, these schemes often adopt discrete Markov chains to model user mobility. Unlike terrestrial users, satellites move at high speeds along predetermined and \textit{repetitive} orbital trajectories~\cite{10436074}, resulting in model merging across regions via model dispersal. Therefore, the proposed FedMeld framework differs from existing distributed FL schemes due to its special circular topology. 
\textbf{Second}, due to computing capabilities and geographical locations, asynchronous FL must be carried out to avoid prolonged client idleness~\cite{10746330,so2022fedspace,9674028,xie2020asynchronous}, which further complicates convergence analysis.
Existing asynchronous FL frameworks, such as FedSat~\cite{9674028}, FedAsync~\cite{xie2020asynchronous}, and AsyncFLEO~\cite{elmahallawy2022asyncfleo}, address model staleness by introducing heuristic weighting or discounting factors to mix current and outdated models, while FedCM accumulates global gradients across rounds~\cite{xu2021fedcm}. Although these approaches improve performance empirically, they lack a rigorous theoretical analysis of how staleness impacts convergence. In contrast, our proposed FedMeld provides a theoretical foundation by deriving a semi-closed-form expression for the optimal mixing ratio, which is intuitively justified and validated through simulations.
To address these challenges, our main contributions are summarized as follows:
\begin{itemize}
\item \textbf{FedMeld with convergence analysis:}
We propose the FedMeld framework for efficient decentralized FL in SGINs. To support effective control, we theoretically characterize the convergence bound of FedMeld under both full and partial participation scenarios, explicitly incorporating the impact of satellite constellation, inter-region model staleness, and model mixing ratio. Our analysis reveals a tradeoff between training latency and model accuracy, highlighting the need for carefully designed control strategies.

\item \textbf{Control of model staleness and mixing ratio:}  
Building on the theoretical insights, we formulate a joint optimization problem to determine the staleness control and mixing ratio (SC-MR) that minimizes the training loss. This optimization problem is decomposed into two sequential subproblems without compromising optimality. We derive the optimal solutions for both decision variables in closed and semi-closed forms, respectively. The results suggest that tighter training latency leads to higher model staleness, while a higher non-IID degree calls for a lower mixing ratio of historical models.

\item \textbf{Experimental results:}  
We perform extensive simulations on diverse datasets to validate the effectiveness of our framework. Compared with existing solutions, FedMeld is not only infrastructure-free but also achieves higher accuracies with reduced training time.  

\end{itemize}

The remainder of this paper is organized as follows. Models and metrics are introduced in Section \ref{sec:model_metric}. The FedMeld algorithm and its convergence analysis are detailed in Section \ref{sec:FedMeld_Alg_Convergence}. Building upon the convergence analysis, the joint optimization problem of SC-MR is formulated, and the optimal solutions are derived in Section \ref{sec:optimization}. Experimental results are provided in Section \ref{sec:experiment}, followed by concluding remarks in Section \ref{sec:conclusion}. 
The main symbols and parameters used in this paper are summarized in Table~\ref{tab:notations} for clarity.

	\begin{table}[t]
		\caption{{Summary of main notations}} \label{tab:notations}
		\vspace{-10pt}
		\begin{center}
			{\footnotesize	\begin{tabular}{|c|p{5cm}|}\hline  
		\textbf{Notation} & \textbf{Definition} \\ \hline
		$ \mathcal{M}, \mathcal{N}$ &  Set of areas and ground edge devices. \\ \hline
        $\mathcal{N}_i$, $\mathcal{U}_{k,i}$ &  Subset of devices in area $i$ and the set of participating clients of area $i$ in the $k$-th global round. \\ \hline
        $R$, $E$ & Total number of training iterations (steps) and the number of local training iterations performed by each client within the same region before local aggregation by its serving satellite. \\ \hline
        $K$ & Number of global training rounds that a serving satellite provides for each region. \\ \hline
        $\varsigma _{t,j}$, $\eta_t$ & Selected data batch from device $j$ at step $t$ and the learning rate at step $t$. \\ \hline
        ${\mathbf{w}}_{t,j}$, ${\mathbf{w}}_t$, ${\mathbf{z}}_t$ & Model parameters of device $j$, global virtual sequence under full participation, and global virtual sequence under partial participation at step $t$. \\ \hline
        $F\left ( \mathbf{w} \right ) $, $F_j\left ( \mathbf{w} \right )$ & Global loss function and local loss function of client $j$. \\ \hline
        $\delta$, $\alpha$ & Global round interval and mixing ratio of model mixing between adjacent regions. \\ \hline
      $T_{k,i }^{{\rm{total}}}$,  $T_{i,i + 1}^{{\rm{idle}}}$, $T_{i,i + 1}^{{\rm{fly}}}$ & Total training latency of area $i$ in global round $k$, total duration from the SCF satellite leaving area $i$ until completing model aggregation for area $\left( i+1\right)$, and idle duration of clients in the region $(i+1)$. \\ \hline
      $T_{\max}$ & Maximum tolerable training time. \\ \hline
      $L, \mu, G, \Gamma$ & Hyperparameter related to assumptions in the convergence analysis. \\ \hline
      $f\left ( \delta, \alpha \right ) $ & Convergence bound of FedMeld as a function of $\delta $ and $\alpha $. \\ \hline
      $\zeta_1, \zeta_2, \zeta_3, \kappa_1, \kappa_2$  & Key parameters related to the convergence bound of FedMeld. \\ \hline
			\end{tabular}}
		\end{center}
	\end{table}

\section{Models and Metrics}\label{sec:model_metric}
We consider an FL framework in SGINs, where a mega-LEO satellite constellation collaborates with groups of terrestrial edge devices to train a neural network. The constellation consists of multiple circular orbits, with satellites uniformly distributed at equal distance intervals within each orbit. These AI-empowered satellites act as parameter servers and move along predetermined orbits. We assume that there are a set $\mathcal{M}$ of areas and a set $ \mathcal{N} $ of ground edge devices, with $ \mathcal{M} = \left\{1, \ldots, M\right\} $ and $ \mathcal{N} = \left\{1, \ldots, N\right\} $. 
These ground areas are distributed along a specific orbit, and
each area $ i \in \mathcal{M} $ covers a subset of devices, denoted by $ \mathcal{N}_i $ with cardinality $ N_i $. 
The learning and communication models are described as follows.


\subsection{Learning Model}
We first introduce the procedure of FedAvg underpinning FedMeld.
The training process consists of $R$ iterations (steps), denoted by $ {\mathcal{R}} = \left\{1, \ldots, R\right\}$. 
We assume that each serving satellite aggregates the model from its covered clients every $E$ iterations, during which all participating clients within the same area perform synchronized local updates.
We define $ {\mathcal{I}}_E = \left\{ {\left. {kE} \right|k = 1,2, \ldots, \left\lfloor {\frac{R}{E}} \right\rfloor } \right\}$ as the set of synchronization steps of each area, where $k$ is the index of the global round.
Since the devices in the same region are located nearby, we assume that they will be handed over to a new satellite parameter server at the same time according to some criteria such as distance and signal-to-noise ratio (SNR).
In addition, considering partial device participation, in the $k$-th global round, the serving (nearest) satellite of area $ i $ randomly selects a device set, $ \mathcal{U}_{k,i} $ (with cardinality $ U_i $), out of $ \mathcal{N}_i $ to participate in the training process.

Each device $ j $ has a private dataset $ {\mathcal{D}}_j $ of size $ \left| {{{\mathcal{D}}_j}} \right| $. 
Denote the $ n $-th training data sample in device $ j $ by $ {\xi  _{j,n}} $, where $ n \in {{{\mathcal{D}}_j}} $.
The training objective of FL is to seek the optimal model parameter $\mathbf{w}^{\ast}$ for the following optimization problem to minimize the global loss function:
\begin{align}
	\mathop {\min }\limits_{\mathbf{w}} F\left( {\mathbf{w}} \right) = \frac{1}{M}\sum\limits_{i \in {\mathcal{M}}} {\frac{1}{{{N_i}}}} \sum\limits_{j \in {{{\cal N}}_i}} {{F_j}\left( {\mathbf{w}} \right)}.
\end{align}
Here, $ {F_j}\left( \cdot \right)  $ is the local loss function, defined by
\begin{equation}
	{F_j}\left( {\mathbf{w}} \right) \triangleq \frac{1}{{\left| {{{\mathcal{D}}_j}} \right|}}\sum\limits_{n \in {{\mathcal{D}}_j}} {\ell}\left( {{\mathbf{w}},{\xi  _{j,n}}} \right),
\end{equation}
where $ {\ell} \left(\cdot, \cdot \right) $ is a user-specified loss function.

Let $ {\mathbf{w}}_{t,j} $ be the model parameters of device $ j $ at step $ t \in {\mathcal{R}}$. 
Consider a global round of the \textit{standard FedAvg} algorithm. First, when $t+1 \notin  \mathcal{I}_E$, the participating clients conduct $E$ local iterations with the latest received model parameters and local data samples. Each local iteration executes ${\mathbf{w}_{t + 1,j}} ={\mathbf{w}_{t,j}} - {\eta _t}{\nabla {F_j}\left( {\mathbf{w}_{t,j},\varsigma _{t,j}} \right)}$, where $ \eta_t $ is the learning rate, $\varsigma_{t,j}  $ is the selected batch from device $ j $ at step $ t $ with $  | \varsigma |$ data samples, and $\nabla {F_j}\left( {\mathbf{w}_{t,j},\varsigma _{t,j}} \right)$ is the stochastic gradient of ${F_j}\left( {\mathbf{w}_{t,j}} \right)$.
Then, when $t+1 \in \mathcal{I}_E$, the serving satellite aggregates the updated model parameters from these clients, conducts FedAvg by executing ${\mathbf{w}_{t + 1,j}} ={\frac{1}{{{N_i}}}\sum\limits_{j \in {{\mathcal{N}}_i}} \left (  {\mathbf{w}_{t,j}} - {\eta _t}{\nabla {F_j}\left( {\mathbf{w}_{t,j},\varsigma _{t,j}} \right)} \right )  }$, and broadcasts the average model back to the clients.

\subsection{Communication Model}
The training process is conducted in consecutive global rounds until the stopping condition is reached. 
Without loss of generality, we focus on the global round $k$ and one ground area $i \in \mathcal{M}$ for analysis and provide detailed latency expressions as follows. 

First, the selected clients conduct $E$ local iterations. For the client $j \in \mathcal{U}_{k,i}$, the computing latency of each local iteration is calculated as
\begin{equation}
	T_j^{\rm{local}}=\frac{ | \varsigma | \Phi}{L_j}, \; \forall j \in \mathcal{N},
\end{equation}
where $ \Phi $ is the computing workload of local iteration and $L_j$ denotes the computing capability (namely the number of floating point operations per second) of client $j$.

Then, clients in each area upload their models to their associated satellites. We consider orthogonal frequency-division multiple access (OFDMA) for uplink transmissions~\cite{9210812,zhang2024lolafl}, and each selected device per global round is assigned a dedicated sub-channel with bandwidth $ W^{\rm{U}} $\footnote{{In addition to digital transmission via OFDMA, analog computation methods like over-the-air aggregation become increasingly advantageous in large-scale systems, as they leverage waveform superposition to aggregate concurrent transmissions without requiring per-device orthogonalization~\cite{10552192,8870236,9272666}.}}.
Then, the uplink rate between device $ j $ and its associated satellite at global round $k$ can be expressed as
\begin{equation}
	{C_{k,j}^{\rm{U}}} =  W^{\rm{U}}{\log _2}\left( {1 + \frac{{{P_{\rm{UE}}}{G_{\rm{UE}}}{G_\mathrm{SAT}} L_{k,j}^{\rm{U,PL}} L^{\mathrm{AL}}}}{{{n_0} W^{\rm{U}}}}} \right),
\end{equation}
where $ P_{\rm{UE}} $ is the transmit power of ground devices, $ G_{\rm{UE}} $ and $ G_{\mathrm{SAT}} $ are transmitting and receiving antenna gain of device and satellite, respectively. The free-space path loss $ L_{k,j}^{\rm{U,PL}} $ is defined as $ L_{k,j}^{\rm{U,PL}} = {\left( {\frac{\lambda_{j}^{\rm{U}} }{{4\pi {d_{k,j}^{\rm{U}}}}}} \right)^2} $, where $ \lambda_{j}^{\rm{U}}  $ is the wavelength of uplink signal and $ d_{k,j}^{\rm{U}}  $ is the uplink transmission distance between device $j$ and its connected satellite. The additional loss $ L^{\mathrm{AL}} $ is used to characterize the attenuation due to environmental factors, and $ n_0 $ is the noise power spectral density.
After the serving satellite conducts FedAvg with latency $T^{\rm{agg}}$, the satellite broadcasts the averaged results to the clients, with the downlink rate between satellite and client $j$ being
\begin{equation}
	{C_{k,j}^{\rm{D}}} =  W^{\rm{D}}{\log _2}\left( {1 + \frac{{{P_{\rm{SAT}}}{G_{\rm{UE}}}{G_\mathrm{SAT}}L_{k,j}^{\mathrm{D,PL}}L^{\mathrm{AL}}}}{{{n_0} W^{\rm{D}}}}} \right),
\end{equation}
where $P_{\rm{SAT}}$ is the transmit power of satellites and $ W^{\rm{D}} $ is the bandwidth that the satellite broadcasts the global model to the users. The free-space path loss of downlink transmission can be calculated by $ L_{k,j}^{\rm{D,PL}} = {\left( {\frac{\lambda_{j}^{\rm{D}} }{{4\pi {d_{k,j}^{\rm{D}}}}}} \right)^2} $, where $ \lambda_{j}^{\rm{D}}  $ is the wavelength of downlink signal and $ d_{k,j}^{\rm{D}}  $ is the downlink transmission distance.
In this paper, we mainly discuss users located in remote areas, where satellite links experience fewer obstacles and weaker multi-path effects compared with urban environments. This assumption is used solely for channel modeling simplification, while the proposed FedMeld framework itself remains applicable to cross-region model aggregation in other urban areas. Therefore, we ignore the impact of small-scale fading caused by the multi-path effect, being consistent with \cite{9174937}.

Given the data rate, the communication latency over uplink and downlink between client $j$ and its associated satellite at the global round $k$ can be expressed as $ T_{k,j}^{\mathrm{U}} = \frac{q}{{C_{k,j}^{\text{U}}}} $ and $ T_{k,j}^{\mathrm{D}} = \frac{q}{{C_{k,j}^{\text{D}}}} $, where $ q $ denotes the model size in bits. 
In addition, the long transmission distance between satellites and ground clients leads that propagation latency cannot be neglected as in terrestrial networks. 
Thus, the propagation latency between device $j$ and its associated satellite at the global round $k$ can be expressed as $T_{k,j}^{\mathrm{P}} = \frac{d_{k,j}^{\mathrm{U}}+d_{k,j}^{\mathrm{D}}}{c}$, where $c$ is the light speed.
Then, the total latency of global round $ k $ can be given by
\begin{align}
	T_{k,i}^{\mathrm{total}} = \mathop {\max }\limits_{j \in {\mathcal{U}_{k,i}}} \left\{ {ET_j^{\mathrm{local}} + T_{k,j}^{\mathrm{U}}} + T_{k,j}^{\mathrm{D}} + T_{k,j}^{\mathrm{P}} \right\} + T_{k,i}^{\rm{agg}},
\end{align}
where $T_{k,i}^{\rm{agg}}$ denotes the computation time required by the serving satellite to perform model aggregation (e.g., FedAvg) over the received client models.

Since we consider the nearest association strategy in the given satellite constellation, the serving duration of each serving satellite over a specific region remains stable \cite{9566290}. 
This coverage duration determines the number of global training rounds the satellite can complete with the clients in a specific region, denoted by $ K $. That is, $K$ is related to the characteristics of the satellite constellation, such as orbital period and satellite density.

\subsection{Performance Metric}
For area $i$, we define an area virtual sequence $ {{{\overline {\mathbf{w}}}_{t,i}}} $ as ${{{\overline {\mathbf{w}}}_{t,i}}} = {\frac{1}{{{N_i}}}\sum\limits_{j \in {{\mathcal{N}}_i}} {{{\mathbf{w}}_{t,j}}} }$. We also define a global virtual sequence as $ {{\overline {\mathbf{w}}}_t} = \frac{1}{M}\sum\limits_{i \in {\mathcal{M}}} {{{\overline {\mathbf{w}}}_{t,i}}} $. When $ t \notin \mathcal{I}_E $, both area and global's sequences are inaccessible. When $ t \in \mathcal{I}_E $, only $ {{{\overline {\mathbf{w}}}_{t,i}}} $ can be fetched for any area $ i \in \mathcal{M} $. After the last training step $ R $, the models can be averaged over all areas to obtain the final global model $ {{\overline {\mathbf{w}}}_R} $.

At the $t$-th iteration, the convergence rate of FedMeld algorithm is defined as the difference between the expectation of the global loss function at the $t$-th step, i.e.,  $F\left( {{{\overline {\mathbf{w}}}_t}} \right)$, and the optimal objective value $F\left( {{{\mathbf{w}}^{ \star }}} \right)$, as
\begin{equation}
    \mathbb{E}\left[ {F\left( {{{\overline {\mathbf{w}}}_t}} \right)} \right] - F\left( {{{\mathbf{w}}^{ \star }}} \right).
\end{equation}
For the general partial participation scheme, the global virtual sequence $ \overline{\mathbf{w}}_t $ is substituted by $ {\overline {\mathbf{z}}_t} = \frac{1}{M}\sum\limits_{i \in {\mathcal{M}}} {\frac{1}{{{U_i}}}\sum\limits_{j \in {{\mathcal{U}}_{t,i}}} {{{\mathbf{w}}_{t,j}}} } $.

\section{FedMeld Algorithm and Convergence Analysis}\label{sec:FedMeld_Alg_Convergence}
In this section, we first describe our proposed FedMeld algorithm. Then, we conduct a convergence analysis of FedMeld under full and partial participation schemes. The results will be used for optimizing the FedMeld in the next section.

\subsection{FedMeld Algorithm}
To realize infrastructure-free model aggregation, FedMeld classifies serving satellites into two functional roles based on their capability to carry and forward models across regions:
\begin{itemize}
    \item \textbf{SCF satellites:} They are responsible for inter-regional model mixing by storing aggregated models from one region and carrying them to the next along their orbital path. Upon arrival at a new region, they mix the stored model from the previous region with the newly aggregated model from local clients, enabling infrastructure-free parameter mixing across adjacent regions.

     \item \textbf{Non-SCF satellites\footnote{{The spatial distribution of ground clusters may be uneven, leading to variations in the flight interval between adjacent regions. 
     The number of global training rounds $K$ for each region is jointly determined by the satellite constellation configuration, the geographical size of each region, the spatial distance between adjacent regions, and the user–satellite communication latency. When two regions are geographically close, the time left for the satellite to serve the next region after completing the current one may be limited. Therefore, $K$ should be carefully designed considering the above factors to ensure that the satellite can complete $K$ rounds within each coverage period. Conversely, when regions are far apart, the idle interval can be effectively utilized by non-SCF satellites, which continue local training before receiving the model from the previous region.}}:} They act as temporary parameter servers that perform local model aggregation (FedAvg) for the current region only. These satellites do not store or forward models across regions and simply discard the aggregated models after completing the local service.
     As a result, their operation introduces asynchronous model updates across regions, since clients in one area can continue local training while awaiting the arrival of an SCF satellite carrying models from previous regions.

\end{itemize}

\begin{remark}
    (\textit{Necessity of Non-SCF Satellites})
    The primary role of non-SCF satellites is to utilize the idle time when an SCF satellite flies from one region to the next. Without non-SCF satellites, clients in a region must wait until the SCF satellite arrives to perform model mixing, resulting in synchronous training and idle periods. In contrast, the presence of non-SCF satellites enables asynchronous model mixing, allowing clients in the next region to continue training and aggregating models before an SCF satellite carrying the model from the previous region arrives.
It is worth noting that designating all satellites as SCF is not desirable. In such a case, multiple outdated models would be repeatedly forwarded and mixed across regions even when fresher models are already available, which diminishes their contribution to convergence.
This SCF-based division aligns with the orbital movement of satellites, where only selected satellites, determined by their orbital paths and the order in which they serve regions, are designated to carry forward models between consecutive regions. 
\end{remark}

For $ t \in \mathcal{I}_E $, there are two cases: 

1) When the serving satellite of area $ i $ is a non-SCF satellite, the devices within area $ i $ upload their models to this satellite for FedAvg.

2) When the serving satellite is an SCF satellite, area $ (i-1) $'s model carried by this satellite will be mixed with area $ i $'s model after performing FedAvg for the collected clients' model in step $t$.

Without loss of generality, we assume that handover occurs after the previous serving satellite broadcasts the latest averaged model to the clients. To accelerate the training process, devices in each region continue additional training rounds with non-SCF satellites before the arrival of the incoming SCF satellite. We define the global round interval of model mixing between adjacent regions as $\delta$, indicating that when model mixing occurs at step $t$, area $i$'s model at step $t$ will be mixed with area $(i-1)$'s model from step $t-\delta E$. 
Since we consider that all the regions are covered by a specific circular orbit, area $i=0$ is equivalent to the last area (i.e., $i=M$), and $i=M+1$ equals the first area (i.e., $i=1$).

Fig. \ref{fig:K_delta} provides an example of how models from four regions are mixed. The procedure of FedMeld is presented below.
\begin{enumerate} 
	\item At the beginning of a mix cycle, users in each area and their serving (nearest) SCF satellite execute $ K $ global rounds cooperatively until the handover happens. Within each global round, the SCF satellite selects a set of participating clients to conduct $ E $ local iterations, collects the uploaded models from these users, runs FedAvg, and broadcasts the averaged model back to users.
 The client selection adopts a random without-replacement strategy, ensuring that each client participates at most once per round and that the global update remains unbiased in expectation.
	
	\item 
	Before the SCF satellite which carries area $\left( i-1 \right) $'s model becomes the serving satellite for area $ i $, users in region $ i $ and their connected non-SCF satellites conduct $ \delta - 1 $ global rounds collaboratively.
	
	\item 
	When the SCF satellite serves area $ i $, users in region $  i  $ conduct $ E $ local iterations and upload their models to this satellite. After running FedAvg for the collected models of area $ i $, the satellite mixes this averaged result with its stored historical model of area $ \left( i-1 \right) $, where the mixing proportion of the historical model is $ \alpha $.
    Let a binary variable, $A_t \in \left\{0,1\right\}$, denote whether the area's model mixes with its adjacent model at step $ t $. When $t \mod (K + \delta) = 0 $, $A_t=1$ holds, otherwise, $A_t =0$.
Therefore, the model of device $ j \in \mathcal{N}_i$ can be updated as follows:
\begin{equation}\label{equ:update_rule_w}
\begin{split}
 & {\mathbf{w}_{t,j}} = \\
 &  \left\{ \begin{array}{l}
		 {\mathbf{w}_{t-1,j}} - {\eta _{t-1}}{\nabla {F_j}\left( {\mathbf{w}_{{t-1},j},\varsigma _{{t-1},j}} \right)} ,  \quad  {\text{if}} \; t \notin {\mathcal{I}}_E, \\
		\left( {1 - \alpha {A_t}} \right){\frac{1}{{{N_i}}}\sum\limits_{j \in {{\mathcal{N}}_i}} \left (  {\mathbf{w}_{t,j}} - {\eta _t}{\nabla {F_j}\left( {\mathbf{w}_{t,j},\varsigma _{t,j}} \right)} \right )  }  \\
        + {\alpha}{A_t}{{\overline {\mathbf{w}}}_{t - {\delta E },i - 1}},  \quad {\text{if}} \; t \in {\mathcal{I}}_E.
	\end{array} \right.  
\end{split}
\end{equation}

	\item Go back to Step 1) until the stopping condition is met.
\end{enumerate}

\begin{figure}[h]
	\centering
	\includegraphics[width = 0.5\textwidth]{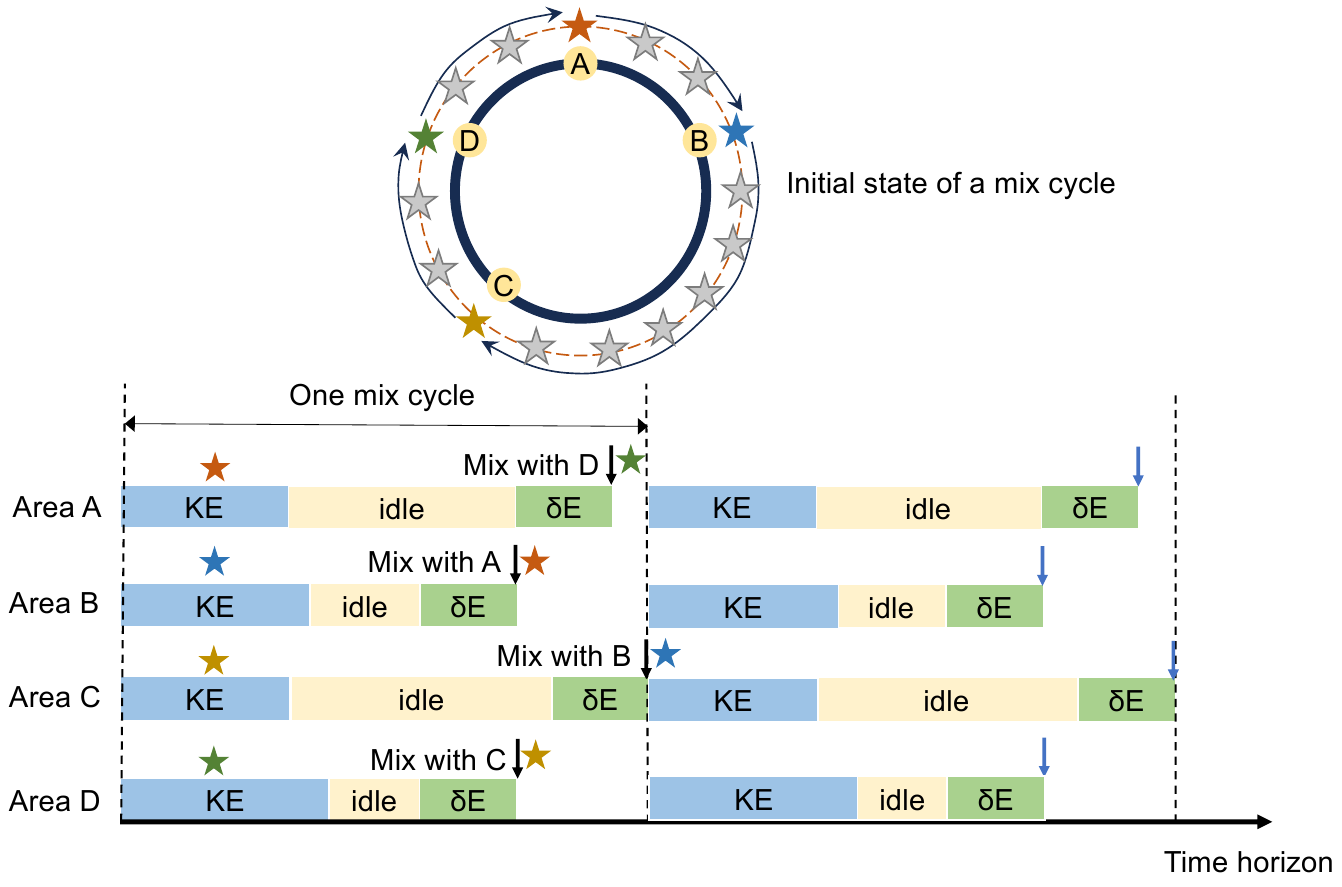}
	\caption{A diagram of parameter mixing in FedMeld algorithm. Colored satellites represent SCF satellites that carry aggregated models across adjacent regions and perform inter-region mixing upon arrival. Gray satellites denote non-SCF satellites, which only perform local aggregation (FedAvg) during their service period without model transfer to subsequent regions.\label{fig:K_delta}}
\end{figure}

The FedMeld algorithm is summarized in Algorithm \ref{alg:FedMeld}.
In FedMeld, each time an SCF satellite arrives at the next region, a parameter mixing occurs with the adjacent regional models. As a result, after $M-1$ mixing, each region incorporates parameter information from all other regions, naturally achieving global model fusion with FedMeld.

\begin{algorithm}
	\caption{FedMeld Algorithm}\label{alg:FedMeld}
	\begin{algorithmic}[1]
		\State \textbf{Input:} Learning rate $\eta_t$, number of local iterations $E$, number of total iterations $R$,  handover condition $A_t$, number of global rounds trained by each SCF satellite $K$, mixing ratio $\alpha$, global round interval of model mixing between adjacent regions $\delta$
		\State \textbf{Output:} Model parameter $\mathbf{w}$
		\State Initialize model parameter $\mathbf{w}$ with $\mathbf{w}_0$
		\For{each global round $k = 1, 2, \ldots, \frac{R}{E}$}
		\For{each serving satellite of area $i = 1 \ldots M$ \textbf{parallel}}
		\State Randomly select a set of active clients $\mathcal{U}_{k,i}$
		\For{each client $j \in \mathcal{U}_{k,i}$ {in parallel}}
		\If{$t \notin {\mathcal{I}}_E$}
		\State{$\mathbf{w}_{t,j} = \mathbf{v}_{t,j}$}
		\Else
		\State{$\mathbf{w}_{t,j}  = \left( {1 - \alpha {A_t}} \right){{\overline {\mathbf{v}}}_{t,i}}  + {\alpha}{A_t}{{\overline {\mathbf{w}}}_{t - {\delta E },i - 1}}$}
		\EndIf
		\EndFor
		\EndFor
		\EndFor
		
		\State \Return $\mathbf{w}$
		
	\end{algorithmic}
\end{algorithm}

\subsection{Convergence Analysis}
Intuitively, FedMeld results in model parameter exchanges across regions by satellite mobility. However, does it converge, and how quickly does it converge? This requires a convergence analysis of FedMeld.

We proceed by making the widely adopted assumptions on client $j$'s local loss function ${F_j}\left( {\bf{w}} \right)$. Let $\nabla {F_j}\left( {\bf{w}} \right) \buildrel \Delta \over = {{\mathbb{E}}_{{\varsigma _j}}}\left[ {\nabla {F_j}\left( {{\bf{w}};{\varsigma _j}} \right)} \right] $.

\begin{assumption}[L-smoothness]\label{ass:smooth}
	For device $j \in {\mathcal{N}}$, ${F_j}\left( {\bf{w}} \right)$ is differentiable and there exists a constant $L \ge 0$ such that for all $ \bf{w}$ and $ \bf{w}'$, ${F_j}\left( \mathbf{w} \right) \le {F_j}\left( \mathbf{w}' \right) + {\left( {\mathbf{w} - \mathbf{w'}} \right)^T}\nabla {F_j}\left( \mathbf{w'} \right) + \frac{L}{2}{\left\| {\mathbf{w} - \mathbf{w'}} \right\|^2}$ holds.
\end{assumption}

\begin{assumption}[$ \mu $-strongly convex]\label{ass:convex}
	For device $j \in {\mathcal{N}}$, ${F_j}\left( \mathbf{w} \right) \ge {F_j}\left( \mathbf{w}' \right) + {\left( {\mathbf{w} - \mathbf{w'}} \right)^T}\nabla {F_j}\left( \mathbf{w'} \right) + \frac{\mu}{2}{\left\| {\mathbf{w} - \mathbf{w'}} \right\|^2}$ holds.
\end{assumption}

\begin{assumption}[Unbiasedness and bounded stochastic gradient variance]\label{ass:unbiased_and_variance}
	For device $j \in {\mathcal{N}}$, the stochastic gradient is unbiased with variance bounded by $\sigma^2$ such that ${\mathbb{E}} \left[\nabla {F_j}\left( {{\mathbf{w}}_j;{\varsigma _j}} \right) \right] = \nabla {F_j}\left( {\mathbf{w}}_j \right)$ and $\mathbb{E}_{\varsigma_j }\left\|\nabla F_j\left({\mathbf{w}}_j ; \varsigma_j\right)-\nabla F_j({\mathbf{w}}_j)\right\|^2 \leq  \sigma_j^2 $ for all $ {\mathbf{w}} $.
\end{assumption}

\begin{assumption}[Bounded gradient]\label{ass:bounded_gradient}
	The expected value of the squared norm of stochastic gradients is consistently bounded: 
	${{{\left\| \nabla {F_j}\left( {{\mathbf{w}}_j;{\varsigma _j}} \right) \right\|}^2}} \leq G^2$ for all $ j \in \mathcal{N} $.
\end{assumption}

We also quantify the degree of non-independent and identically distributed (non-IID) based on the following definition.
\begin{definition}[Degree of non-IID]
	Let $F^*$ and $F_j^*$ be the minimum values of $F$ and $F_j$, respectively. We use a constant $ \Gamma $, denoted by $\Gamma  = {F^*} - \frac{1}{M}\sum\limits_{i \in {\mathcal{M}}} \frac{1}{N_i} {\sum\limits_{j \in {\mathcal{N}}_i} {F_j^*} } $, to quantify the non-IID degree of data in SGINs. If the data are IID, $\Gamma$ goes to zero when the number of data increases. 
If the data are non-iid, we will have $\Gamma >0$ and its magnitude shows the data heterogeneity.
\end{definition}

Based on the assumptions above, we first analyze the convergence bound of FedMeld under the full participation scheme and then extend the result to the partial participation one.
For analysis, we introduce an auxiliary variable $ {\mathbf{v}_{t + 1,j}} $ to denote the intermediate result of one step stochastic gradient descent (SGD) update from $ {\mathbf{w}_{t,j}} $. That is, 
\begin{equation}
    {\mathbf{v}_{t+1,j}} = {\mathbf{w}_{t,j}} - {\eta _t}{\nabla {F_j}\left( {\mathbf{w}_{t,j},\varsigma _{t,j}} \right)}.
\end{equation}
Then, the update rule of the model parameter of client $j$ under FedMeld can be rewritten as follows:
\begin{equation}
    {\mathbf{w}_{t,j}} = \begin{cases}
 {\mathbf{v}_{t,j}} & \text{ if } t \notin  \mathcal{I}_E  \\
\left( {1 - \alpha {A_t}} \right){{\overline {\mathbf{v}}}_{t,i}}  + {\alpha}{A_t}{{\overline {\mathbf{w}}}_{t - {\delta E },i - 1}},  & \text{ if } t \in \mathcal{I}_E
\end{cases}.
\end{equation}
We also define a virtual sequence as $ {{\overline {\mathbf{v}}}_t} = \frac{1}{M}\sum\limits_{i \in {\mathcal{M}}} {{{\overline {\mathbf{v}}}_{t,i}}}= \frac{1}{M}\sum\limits_{i \in {\mathcal{M}}} {\frac{1}{{{N_i}}}\sum\limits_{j \in {{\mathcal{N}}_i}} {{{\mathbf{v}}_{t,j}}} } $, where $ {{{\overline {\mathbf{v}}}_{t,i}}} $ is area $ i $'s virtual sequence.

To facilitate the subsequent convergence proof, we first recall a classical result of one-step SGD, which serves as the foundation for our later derivations.

\begin{lemma}\label{lemma:one_step_SGD}
(Results of one step SGD \cite[Lemma 1]{Li2020n})
	Let Assumption \ref{ass:smooth} and \ref{ass:convex} hold. 
	Notice that $ {{\overline {\mathbf{v}}}_{t + 1}} = {{\overline {\mathbf{w}}}_t} - \eta_t {\mathbf{g}_t} $ always holds.
	If $ \eta_t \leq \frac{1}{4L} $,  we have
	\begin{equation}\label{equ:one_step_SGD}
		\begin{split}
			&	{\left\| {{{\overline {\mathbf{v}}}_{t + 1}} - {{\mathbf{w}}^{ \star }}} \right\|^2} 	\leq \left( {1 - \mu {\eta _t}} \right)\mathbb{E}{\left\| {{{\overline {\mathbf{w}}}_t} - {{\mathbf{w}}^{ \star }}} \right\|^2} + \eta _t^2\mathbb{E}{\left\| {{{\mathbf{g}}_t} - {{\overline {\mathbf{g}}}_t}} \right\|^2}   \\
			& + 6L\eta _t^2\Gamma + 2\mathbb{E}\sum\limits_{i \in {\mathcal{M}}} {\sum\limits_{j \in {\mathcal{N}_i}} {\frac{1}{{M{N_i}}}} } {\left\| {  {\mathbf{w}_{t,j}} - {{\overline {\mathbf{w}}}_t}} \right\|^2},
		\end{split}	
	\end{equation}
	where $ {\mathbf{g}_t} = \frac{1}{M}\sum\limits_{i \in {\mathcal{M}}} {\frac{1}{{{N_i}}}\sum\limits_{j \in {\mathcal{N}_i}} {\nabla {F_j}\left( {\mathbf{w}_{t,j},\varsigma _{t,j}} \right)} } $, $ {{\overline {\mathbf{g}}}_t}  = \frac{1}{M}\sum\limits_{i \in {\mathcal{M}}} {\frac{1}{{{N_i}}}\sum\limits_{j \in {\mathcal{N}_i}} {\nabla {F_j}\left( {{\mathbf{w}_{t,j}}} \right)} } $, and $ \Gamma = {F^*} - \sum\limits_{i \in {\mathcal{M}}} {\sum\limits_{j \in {\mathcal{N}_i}} {\frac{1}{{M{N_i}}}} } F_j^* $.
\end{lemma}

With Lemma \ref{lemma:one_step_SGD}, we have the following theorem to provide the convergence bound of FedMeld under full participation scheme.

\begin{theorem}[Convergence bound under full participation scheme]\label{theorem:convergence_bound}
	Choose the learning rate $ \eta_{t} = \frac{2}{\mu (\gamma + t)}  $ for $ \gamma  = \max \left\{ {\frac{{8L}}{\mu },E} \right\}-1 $,
	when $ \delta E \le \frac{{m\left( \alpha  \right){{\left( {t + \gamma  + 1} \right)}^2}}}{{{{\left( {t + \gamma } \right)}^2} + m\left( \alpha  \right)\left( {t + \gamma  + 1} \right)}} $, the convergence bound with full participation at the last time step $ R $ satisfies
	\begin{equation}
		\begin{split}
			 \mathbb{E}\left[ {F\left( {{{\overline {\mathbf{w}}}_R}} \right)} \right] &- F\left( {{{\mathbf{w}}^{ \star }}} \right)  \leq\left( {\frac{1}{{1 - \alpha }} - \alpha + \frac{1}{2}  } \right)
			\frac{{L}}{{ \mu \left( {R + \gamma  } \right)}} \\
			&\cdot \left[ {\frac{{2B}}{{{\mu}}} + \frac{{\mu \left( {\gamma  + 1} \right)}}{2}\mathbb{E}{{\left\| {{{\overline {\mathbf{w}}}_1} - {{\mathbf{w}}^{ \star }}} \right\|}^2}} \right] + L{Q_{R - 1}},
		\end{split}
	\end{equation}
	where $ m\left(\alpha\right) = {\frac{{4\left( {1 - \alpha } \right)\left( {2{\alpha ^2} - \alpha  + 1} \right)}}{{\alpha \left( {2{\alpha ^2} - 3\alpha  + 3} \right)}}} $, $ {Q_{R - 1}} = \mathbb{E}\sum\limits_{i \in {\mathcal{M}}} {\sum\limits_{j \in {{\mathcal{N}}_i}} {\frac{1}{{M{N_i}}}} } {\left\| {  {\mathbf{w}_{{R - 1},j}} - {{\overline {\mathbf{w}}}_{R - 1}}} \right\|^2} $ and $ B = \sum\limits_{i \in {\mathcal{M}}} {\sum\limits_{j \in {{\mathcal{N}}_i}} {\frac{{\sigma _j^2}}{{{{\left( {M{N_i}} \right)}^2}}}} }  + 6L\Gamma  $.
\end{theorem}

\begin{proof}
	Please see Appendix \ref{proof:theorem_convergence_bound}.
\end{proof}

The constraint $ \delta E \le \frac{{m\left( \alpha  \right){{\left( {t + \gamma  + 1} \right)}^2}}}{{{{\left( {t + \gamma } \right)}^2} + m\left( \alpha  \right)\left( {t + \gamma  + 1} \right)}} $ in Theorem \ref{theorem:convergence_bound} imposes an upper bound on the round interval $\delta$, which quantifies the staleness of models exchanged between adjacent regions. Practically, this ensures that the model carried by the SCF satellite is not excessively outdated when mixed with the latest model in the next region. If $\delta$ exceeds this limit, the stale models will lead to performance degradation and hinder convergence.

To analyze the evolution of model divergence across clients and clusters, we define $ Q_t =  \mathbb{E}\sum\limits_{i \in {\mathcal{M}}} {\sum\limits_{j \in {\mathcal{N}_i}} {\frac{1}{{M{N_i}}}} } {\left\| {  {\mathbf{w}_{t,j}} - {{\overline {\mathbf{w}}}_t}} \right\|^2}$,
which measures the expected variance of local models with respect to the global mean at step $t$. This quantity serves as a key indicator in the subsequent analysis.
With this definition, we next establish an upper bound for $Q_t$, which forms the basis for the overall convergence proof.

\begin{lemma}\label{lemma:relationship_Qt}
	For any $ t \geq 0 $ satisfying $ t \in \mathcal{I}_E $, $ Q_t$ has the following upper bound,
	\begin{equation}\label{equ:relationship_Qt}
		\begin{split}
			{Q_t}&  \leq \left( {1 + b} \right){\left( {1 - \alpha {A_t}} \right)^2}{Q_{t-E}} + \left( {1 + b} \right){\left(\alpha {A_t}\right) ^2}{Q_{t - \delta E }} \\
			& + 4\left( {1 + \frac{1}{b}} \right){\left( {1 - \alpha {A_t}} \right)^2}{\left(E-1\right)^2}{G^2}\eta _t^2,
		\end{split}
	\end{equation}
	For $ t \notin \mathcal{I}_E $, we can always find a $ \tilde{t} $ such that  $ \tilde{t} = \left\lfloor {\frac{t}{E}} \right\rfloor E $. Then, $ Q_t $ has the following upper bound,
	\begin{equation}\label{equ:Q_t_noagg}
		Q_t \leq \left(1+b\right) Q_{\tilde{t}} +  4\left( {1 + \frac{1}{b}} \right){\left(E-1\right)^2}{G^2}\eta _t^2.
	\end{equation}
	Here, $b$ is a positive constant, and its value will be discussed in detail in the subsequent paragraphs.
\end{lemma}

\begin{proof}
	Please see Appendix \ref{proof:lemma_relationship_Qt}.
\end{proof}

Based on the recursive relationship in Lemma \ref{lemma:relationship_Qt}, we can now derive an upper bound of $Q_{R-1}$, as summarized in Theorem \ref{theorem:Q_T_UpperBound}.

\begin{theorem}\label{theorem:Q_T_UpperBound}
	Assume that $ {R-1} $ is the final step of a particular mix cycle.
	The upper bound of $ Q_{R-1} $ satisfies
	\begin{equation}
		{Q_{R-1}} \leq   \eta _E^2\frac{{{\kappa _2}}}{{1 - {\kappa _1}}},
	\end{equation}
	where $ \kappa_1 = {\left( {1 + b} \right)^{ K + \delta }  }{\left( {1 - \alpha } \right)^2} + \left( {1 + b} \right){\alpha ^2} $ and 
	$ {\kappa _2} = 4\left( {1 + \frac{1}{b}} \right){\left( {E - 1} \right)^2}{G^2}\left[ {\frac{{{\kappa _1}}}{b} + {{\left( {1 - \alpha } \right)}^2}} \right] $. Here, $ \kappa_1 < 1 $.
\end{theorem}
\begin{proof}
	Please see Appendix \ref{proof:theorem_Q_T_UpperBound}.
\end{proof}

Having established the convergence bound under full participation, we extend our analysis to the more practical case of partial client participation.
\begin{lemma}\label{lemma:unbiased_sampling}
	(Bounding average variance under unbiased sampling scheme \cite[Lemma 1]{10082940}) If $t $ is a global round and satellites select clients randomly without replacement, then $\mathbb{E}_{\mathcal{U}_t} \overline{\mathbf{z}}_t = \overline{\mathbf{w}}_t$, where $\mathcal{U}_t = \bigcup_{i \in \mathcal{M}} \mathcal{U}_{t,i}$ represents the union of selected clients. Due to partial participation, we have
	\begin{equation}
		{\mathbb{E}_{{{\mathcal{U}}_t}}}{\left\| {{{\overline {\mathbf{z}}}_t} - {{\overline {\mathbf{w}}}_t}} \right\|^2} \leq \frac{1}{{{M^2}}}\sum\limits_{i \in {\mathcal{M}}} {{\frac{{{N_i} - {U_i}}}{{{U_i}{N_i}\left( {{N_i} - 1} \right)}}}} \sum\limits_{j \in {{\mathcal{N}}_i}} {{{\left\| {{\mathbf{w}_{t,j}} - {{\overline {\mathbf{w}}}_t}} \right\|}^2}}.
	\end{equation}
\end{lemma}

For notational brevity,
we denote constants $ {\zeta _1}, {\zeta _2}, {\zeta _3} $ as follows,
\[\left\{ \begin{array}{l}
	{\zeta _1} = \frac{L}{{\mu \left( {R + \gamma } \right)}}\left[ {\frac{{2B}}{{{ \mu}}} + \frac{{\mu \left( {\gamma  + 1} \right)}}{2}\mathbb{E}{{\left\| {{{\overline {\mathbf{w}}}_1} - {{\mathbf{w}}^{ \star }}} \right\|}^2}} \right],\\
	{\zeta _2} = L\left[ {\frac{{1 + b}}{2}\mathop {\max }\limits_{i \in \mathcal{M}} \left\{ {\frac{{{N_i} - {U_i}}}{{M{U_i}\left( {{N_i} - 1} \right)}}} \right\} + 1} \right],\\
	{\zeta _3} = 2L\left( {1 + \frac{1}{b}} \right){\left( {E - 1} \right)^2}{G^2}\eta _R^2\mathop {\max }\limits_{i \in \mathcal{M}} \left\{ {\frac{{{N_i} - {U_i}}}{{M{U_i}\left( {{N_i} - 1} \right)}}} \right\}.
\end{array} \right.\]

Using the above Lemma \ref{lemma:unbiased_sampling}, the convergence bound under the partial participation scheme can be derived as follows.

\begin{theorem}[Convergence bound under partial participation]\label{theorem:convergence_bound_partial}
	With an unbiased sampling scheme, the convergence bound with partial client participation at step $ R $ can be expressed as 
	\begin{align}\label{equ:upperbound_partial}
		\resizebox{1\hsize}{!}{$	 \mathbb{E}\left[ {F\left( {{{\overline {\mathbf{z}}}_R}} \right)} \right] - F\left( {{{\mathbf{w}}^{ \star }}} \right)  \leq  {\zeta _1}\left( {\frac{1}{{1 - \alpha }} - \alpha + \frac{1}{2} } \right) + {\zeta _2}\frac{{{\kappa _2}}}{{1 - {\kappa _1}}} + {\zeta _3}. $}
	\end{align}
\end{theorem}
\begin{proof}
	Please see Appendix \ref{proof:convergence_bound_partial}.
\end{proof}

Building upon the convergence result established in Theorem \ref{theorem:convergence_bound_partial}, we next provide several remarks to interpret its implications and practical insights.

\begin{remark}[Impact of partial participation]
	The partial participation scheme will bring an additional term reflected by $ \mathop {\max }\limits_{i \in \mathcal{M}} {\frac{{{N_i} - {U_i}}}{{M{U_i}\left( {{N_i} - 1} \right)}}} $. When $ U_i $ is larger, the negative effect caused by this term will be mitigated. If all clients join the training process, this term will be eliminated.
\end{remark}

\begin{remark}[Discussion of $ b $ in Theorem \ref{theorem:Q_T_UpperBound}]
	It is clear that $ \kappa _1 $ has the following upper bound, i.e., $ {\kappa _1} \leq {\left( {1 + b} \right)^{K + \delta }}{c^2} \cdot \max \left\{ {{\alpha ^2},{{\left( {1 - \alpha } \right)}^2}} \right\} $, where $ c \geq 1$. 
	Define an auxiliary variable $ \rho $ satisfying $ 1 < \rho  \leq \min \left\{ {\frac{1}{\alpha },\frac{1}{{1 - \alpha }}} \right\} $.
	When choosing $ b = {\rho ^{\frac{1}{{K + 1}}}} - 1 $ and $ c = \frac{1}{{\rho  \cdot \max \left\{ {\alpha ,1 - \alpha } \right\}}} \geq 1 $ for $ \delta  < K + 2 $, $ \kappa_1 < 1$ can be guaranteed. 
	Then, $ \kappa_1 $ and $ \kappa_2 $ can be recast as 
	\begin{equation}
	\resizebox{1\hsize}{!}{$		\left\{
		\begin{aligned}
			\kappa_1 &= \rho^{\frac{K+\delta}{K+1}} (1-\alpha)^2 + \rho\alpha^2, \\
			\kappa_2 &= 4{\left( {E - 1} \right)^2}{G^2}\frac{{{\rho ^{\frac{1}{{K + 1}}}}}}{{{\rho ^{\frac{1}{{K + 1}}}} - 1}}  \left[ {\frac{{{\rho ^{\frac{{K + \delta }}{{K + 1}}}}{{\left( {1 - \alpha } \right)}^2} + \rho {\alpha ^2}}}{{{\rho ^{\frac{1}{{K + 1}}}} - 1}} + {{\left( {1 - \alpha } \right)}^2}} \right].
		\end{aligned}
		\right. $}
	\end{equation}
	For simplicity, $ \rho $ can be set as $ \frac{3}{2} $. 
\end{remark}

\begin{remark}[Impact of $\delta$ and $\alpha$ on model accuracy]
Given the communication scenario and initial training setting in SGINs, $E$, $G$, and $K$ (since the training rounds of each serving satellite are stable in the long term) are fixed. Therefore, the upper bound of the training loss in (\ref{equ:upperbound_partial}) can be adjusted by designing parameters $\delta$ and $\alpha$. We can find that the right side of (\ref{equ:upperbound_partial}) is a monotonically increasing function of $\delta$, indicating the tradeoff between latency and accuracy. However, the coupling relationship between $\delta$ and $\alpha$ not only exists in the upper bound of the training loss but also in the constraint presented in (\ref{theorem:convergence_bound}). Their complex dependency makes it challenging to derive their joint optimal solutions, motivating us to formulate an optimization problem and design an algorithm in the next section.
\end{remark}

\section{FedMeld Optimization}\label{sec:optimization}
In this section, the results from the preceding convergence analysis are applied to the FedMeld optimization.
First, we formulate a joint optimization problem of SC-MR to minimize the upper bound of the training loss under FedMeld within a fixed span of training time. Then, we demonstrate that the problem can be decomposed into two consecutive subproblems, which can be optimally solved using the algorithm we proposed.

\subsection{Problem Formulation}
To ensure convergence performance, we formulate an optimization problem to find the optimal global round interval of model mixing $\delta$ and model mixing ratio $\alpha$, which is cast as follows:
\begin{subequations}
	\begin{equation}\label{P0_obj:min_bound}
		\mathcal{P}1:	\mathop {\min }\limits_{\delta ,\alpha } \; f\left(\delta, \alpha\right) = {\zeta _1}\left( {\frac{1}{{1 - \alpha }} - \alpha + \frac{1}{2} } \right) + {\zeta _2}\frac{{{\kappa _2}}}{{1 - {\kappa _1}}} + {\zeta _3}
	\end{equation}
	\begin{equation}\label{P0_cst:fly_time}
		\begin{split}
			{\rm{s.t.}} \quad T_{i,i + 1}^{{\rm{idle}}} +& \sum\limits_{k' = k + 1}^{k + \delta } {T_{k',i + 1}^{{\rm{total}}}}  = T_{i,i + 1}^{{\rm{fly}}},\\
			& \forall k: k-K \bmod (K+\delta)=0, i \in {{\cal M}},
		\end{split}
	\end{equation}
	\begin{equation}\label{P0_cst:max_mix_time}
		\frac{{R - 1}}{{\left( {K + \delta } \right)E}} \cdot \mathop {\max }\limits_{i \in \mathcal{M}} \left\{ {T_{i,i + 1}^{{\rm{fly}}}} \right\} \le {T_{\max }},
	\end{equation}
	\begin{equation}\label{P0_cst:upper_deltaE}
		\delta E \le \frac{{m\left( \alpha  \right){{\left( {t + \gamma  + 1} \right)}^2}}}{{{{\left( {t + \gamma } \right)}^2} + m\left( \alpha  \right)\left( {t + \gamma  + 1} \right)}}, \; \forall t \in {\mathcal{R}},
	\end{equation}
\end{subequations}
where the objective (\ref{P0_obj:min_bound}) is to optimize the convergence bound. Constraint (\ref{P0_cst:fly_time}) ensures that the sum of idle duration and additional training time equals the satellite flight time between two adjacent areas. 
During the stage from when the SCF satellite leaves the visible range of area $ i $ until completing model aggregation for area $\left( i+1 \right)$, we define two parameters: $ T_{i,i + 1}^{{\rm{fly}}} $ and $T_{i,i+1}^{{\rm{idle}}}$. Here, $ T_{i,i + 1}^{{\rm{fly}}} $ is the total duration of this process, which can be determined using publicly available Two-Line Element sets (TLEs), and $T_{i,i+1}^{{\rm{idle}}}$ is the idle duration of clients in the region $(i+1)$, i.e., they are neither training with SCF satellites nor non-SCF satellites. 
Then, we have $ T_{i,i + 1}^{{\rm{idle}}} + \sum\limits_{k' = k+1}^{k + \delta } {T_{k',i + 1}^{{\rm{total}}}}  = T_{i,i + 1}^{{\rm{fly}}} $ for any $ k-K \bmod (K+\delta)=0 $.
Constraint (\ref{P0_cst:max_mix_time}) ensures that the total training time does not exceed the maximum constraint. Constraint (\ref{P0_cst:upper_deltaE}) limits the maximum round gap of the mixed models, which has been discussed in Theorem \ref{theorem:convergence_bound}. 

\subsection{Equivalent Problem Decomposition}
The joint SC-MR problem can be decoupled into two sequential subproblems.
The first subproblem SC is to minimize the model staleness with relevant constraints from $ \mathcal{P}1 $, which can be expressed as
\begin{subequations}
	\begin{equation*}
		({\mathcal{P}}2:{\text{SC}}) \quad	\min \; {\delta}
	\end{equation*}
	\begin{equation*}
		\mathrm{s.t.} \;
		{\text{(\ref{P0_cst:fly_time}) \&  (\ref{P0_cst:max_mix_time})}}.
	\end{equation*}
\end{subequations}
Given the optimal model staleness indicator $ \delta^* $ from solving $ \mathcal{P}2 $, the second subproblem MR can be simplified from $ \mathcal{P}1 $ as 

\begin{subequations}
	\begin{equation*}
		({\mathcal{P}}3:{\text{MR}})	\quad	\min \limits_{\alpha } \; f\left( {{\delta ^*},\alpha } \right)
	\end{equation*}
	\begin{equation*}
		\mathrm{s.t.} \; \delta^* E \le \frac{{m\left( \alpha  \right){{\left( {t + \gamma  + 1} \right)}^2}}}{{{{\left( {t + \gamma } \right)}^2} + m\left( \alpha  \right)\left( {t + \gamma  + 1} \right)}}, \; \forall t \in {\mathcal{R}}.
	\end{equation*}
\end{subequations}
The optimality of the aforementioned decomposition is demonstrated in the following lemma.

\begin{lemma}\label{lemma:subsequent}
	Solving $ \mathcal{P}1 $ is equivalent to first solving $ \mathcal{P}2 $ and then solving $ \mathcal{P}3 $.
\end{lemma}
\begin{proof}
	The feasibility of $ \mathcal{P}1 $ indicates the right side of (\ref{P0_cst:upper_deltaE}) is large enough.
	It is clear that the objective function $ f\left(\delta, \alpha\right) $ of $ \mathcal{P}1 $ is monotonically increasing with respect to $ \delta $. 
	Simultaneously, the right side of (\ref{P0_cst:upper_deltaE}) is a monotonically decreasing function with respect to $ \alpha $. 
	Thus, solving $ \mathcal{P}2 $, which can obtain the minimum value of $ \delta $ (i.e., $ \delta^* $) under its constraints, leads to the largest feasible range for $ \alpha $, which must contain the optimal solution of $ \alpha $ (i.e., $\alpha^*$). Obviously, plugging $\delta^*$ into $ \mathcal{P}3 $ and then solving $ \mathcal{P}3 $ leads to the global optimal solution.
\end{proof}

\subsection{Optimal Solutions of Joint SC-MC Design}
\subsubsection{Optimal Solution of $ \delta $}
Since only Constraint (\ref{P0_cst:max_mix_time}) limits the lower bound of $ \delta $, the optimal $ \delta $ denoted by $ \delta^{\ast} $ is given by
\begin{align}\label{equ:optimal_delta}
	\delta  \geq  \delta^{\ast}  = \frac{{R - 1}}{{E{T_{\max }}}} \cdot \mathop {\max }\limits_i \left\{ {T_{i,i + 1}^{{\rm{fly}}}} \right\} - K.
\end{align}
With (\ref{equ:optimal_delta}), we have the following observation.
\begin{remark}[Discussion of the optimal model staleness indicator]
 When the training latency requirements become more stringent, or the satellite distribution becomes denser, $\delta^{\ast}$ increases. This indicates that the difference in the global round interval of model mixing between adjacent regions becomes larger.
\end{remark}

\subsubsection{Optimal Solution of $ \alpha $}
First, we discuss the feasible set of $ \alpha $. 
By taking the derivative, we can find that if the monotonic decreasing function $ m\left(\alpha\right) $ satisfies $ m\left(\alpha\right) \geq 2 $, i.e., $ \alpha \leq \frac{1}{2} $, the right side of (\ref{P0_cst:upper_deltaE}) is a non-decreasing function of $ t $ . 
Therefore, as long as $ {\delta ^*}E \le \frac{{m\left( \alpha  \right){{\left[ {\left( {K + {\delta ^*}} \right)E + \gamma } \right]}^2}}}{{{{\left[ {\left( {K + {\delta ^*}} \right)E + \gamma  - 1} \right]}^2} + m\left( \alpha  \right)\left[ {\left( {K + {\delta ^*}} \right)E + \gamma } \right]}} $ holds, 
which is equivalent to 
\begin{equation}\label{cst:alpha_delta*}
	m\left( \alpha  \right) \ge \frac{{{\delta ^*}E{{\left[ {\left( {K + {\delta ^*}} \right)E + \gamma  - 1} \right]}^2}}}{{\left[ {\left( {K + {\delta ^*}} \right)E + \gamma } \right]\left( {KE + \gamma } \right)}},
\end{equation}
then (\ref{P0_cst:upper_deltaE}) can be always satisfied.
Let $ \tilde{\alpha}$ be the value of $ \alpha $ which makes (\ref{cst:alpha_delta*}) an equality, i.e., $m\left(\tilde{\alpha} \right) = \frac{{{\delta ^*}E{{\left[ {\left( {K + {\delta ^*}} \right)E + \gamma  - 1} \right]}^2}}}{{\left[ {\left( {K + {\delta ^*}} \right)E + \gamma } \right]\left( {KE + \gamma } \right)}}$. 
Thus, the feasible set of $ \alpha $ is $ \alpha \in \left[0, \min \left\{\tilde{\alpha}, \frac{1}{2} \right\}  \right] $.

Then, we analyze the derivative of $ f\left( {{\delta ^*},\alpha } \right) $ and obtain the optimal solution of $ \alpha $.
By taking the derivative, we have
\begin{equation}\label{equ:derivative_f}
	f^{'}\left(\alpha\right)=	\frac{{\partial f\left( {{\delta ^*},\alpha } \right)}}{{\partial \alpha }} = {\zeta _1}\left[ {\frac{1}{{{{\left( {1 - \alpha } \right)}^2}}} - 1} \right] + {\zeta _2}\frac{{{g_1}\left( \alpha  \right)}}{\left[{{g_2}\left( \alpha  \right)}\right]^2},
\end{equation}
where ${g_1}\left( \alpha  \right) =  - \rho {D_2}{\alpha ^2} + \left[ {{\left( {{\rho ^{\frac{{K + \delta^* }}{{K + 1}}}} + \rho } \right){D_1}} + \left( {\rho  + 1} \right){D_2}} \right]\alpha - \left( {{\rho ^{\frac{{K + \delta^* }}{{K + 1}}}}{D_1} + {D_2}} \right)$
and ${g_2}\left( \alpha  \right) = { {\left( {{\rho ^{\frac{{K + \delta^* }}{{K + 1}}}} + \rho } \right){\alpha ^2} - 2{\rho ^{\frac{{K + \delta^* }}{{K + 1}}}}\alpha  + {\rho ^{\frac{{K + \delta^* }}{{K + 1}}}} - 1} }$.
Here, we use $ {D_1} = 4{\left( {E - 1} \right)^2}{G^2}\frac{{{\rho ^{\frac{1}{{K + 1}}}}}}{{{{\left( {{\rho ^{\frac{1}{{K + 1}}}} - 1} \right)}^2}}} $ and $ {D_2} = 4{\left( {E - 1} \right)^2}{G^2}\frac{{{\rho ^{\frac{1}{{K + 1}}}}}}{{{\rho ^{\frac{1}{{K + 1}}}} - 1}} $ for notation simplicity.
Note that $ g_1\left(\alpha\right) $ is monotonically increasing when $ \alpha $ is no more than $ \frac{{\left( {{\rho ^{\frac{{K + \delta^* }}{{K + 1}}}} + \rho } \right){D_1} + \left( {\rho  + 1} \right){D_2}}}{{2\rho {D_2}}}$, which is less than $ \frac{1}{2} $. 
Specifically, $ {g_1}\left( 0 \right) < 0 $, $ {g_1}\left( {\frac{1}{2}} \right) < 0  $ and $ {g_1}\left( 1 \right) > 0 $ always hold.
For $ g_2\left(\alpha\right) $, its discriminant is negative, denoting that $ g_2\left(\alpha\right) $ is always positive.
In addition, $ g_2\left(\alpha\right) $ is monotonically decreasing when $ \alpha \leq \frac{{{\rho ^{\frac{{K + \delta^* }}{{K + 1}}}}}}{{{\rho ^{\frac{{K + \delta^* }}{{K + 1}}}} + \rho }} $, where $ \frac{{{\rho ^{\frac{{K + \delta^* }}{{K + 1}}}}}}{{{\rho ^{\frac{{K + \delta^* }}{{K + 1}}}} + \rho }}  $ falls in the range of $ \left[ {\frac{1}{2}, 1} \right) $. The characteristics of $ g_1\left(\alpha\right) $ and $ g_2\left(\alpha\right) $ show that $ \frac{{{g_1}\left( \alpha  \right)}}{\left[{{g_2}\left( \alpha  \right)}\right]^2} $ is monotonically increasing when $ \alpha \in \left[0,\frac{1}{2}\right] $. Simultaneously, $ {\frac{1}{{{{\left( {1 - \alpha } \right)}^2}}} - 1} $ is a monotonically increasing function of $ \alpha $. Thus, we can say that derivative function $f^{'}\left(\alpha\right) $ is non-decreasing with respect to $ \alpha \leq \frac{1}{2}$.
Combining with the constraint (\ref{cst:alpha_delta*}), the optimal solution $ \alpha^{*} $ can be determined by
\begin{equation}
	{\alpha ^ * } = \min \left\{ {\dot \alpha ,\tilde \alpha } \right\},
\end{equation}
where
\begin{equation}
	{\dot \alpha } = \left\{ \begin{array}{l}
		\alpha_0 :{f^{'}}\left( \alpha_0  \right) = 0, \quad {\rm{if}} \; {f^{'}}\left( {\frac{1}{2}} \right) > 0, \\
		\frac{1}{2}, \quad  {\rm{if}} \; {f^{'}}\left( {\frac{1}{2}} \right) \leq 0.
	\end{array} \right.
\end{equation}
Here, $ \alpha_0 $ can be obtained by bisection method since the $f^{'}\left(\alpha\right) $ is monotonic and thus there is only one zero root if $ {f^{'}}\left( {\frac{1}{2}} \right) > 0 $.

\begin{remark}[Discussion of the optimal mixing ratio]\label{remark:optimal_mixing_ratio}
	The large non-IID degree of data $ \Gamma $ leads to bigger $ \zeta_1 $. When $ \zeta_1 $  is large enough, $ \alpha_0 $ will be smaller, making $ \alpha^* < \frac{1}{2} $ possible. This implies that when the non-IID degree of data is higher, the mixing ratio of historical models from adjacent areas should be less.
   Intuitively, when regional data distributions are highly heterogeneous, the historical model parameters from another region are not only stale but also biased toward a divergent data domain. Assigning a large weight to such a model may distort the optimization trajectory of the global model, resulting in slower or unstable convergence. Therefore, a smaller $\alpha$ helps reduce the influence of stale or biased historical models, thereby enhancing the stability of convergence.
\end{remark}

\section{Experimental Results}\label{sec:experiment}
In this section, we conduct experiments to evaluate the proposed FedMeld in terms of model accuracy and time consumption in comparison with other baselines.

\subsection{Experiment Settings}
In the simulations, we examine the SGINs consisting of an LEO constellation and $N=40$ ground devices. 
The constellation follows the Starlink system, consisting of 24 orbital planes with 66 satellites in each orbit, while the ground devices are distributed across $M=8$ clusters.

Similar to previous works \cite{10436088,9982621,10121575,wei2025pipelining}, 
we conduct image classification tasks using the CIFAR-10 \cite{krizhevsky2009learning} and MNIST \cite{726791} datasets, employing the ResNet-18 architecture \cite{7780459}. 
Given that image classification is a key task in remote applications like remote sensing and healthcare, these datasets provide a relevant and practical basis for validating our proposed FedMeld. The CIFAR-10 dataset comprises 50,000 training samples and 10,000 testing samples across 10 categories, while MNIST contains 60,000 training samples and 10,000 testing samples with labels ranging from 0 to 9.
Each client possesses an equal number of images, with distinct images allocated to each device. 
The computational capacity of each client is $h_j = 15.11$ TFLOPS. In each global round, 80\% of clients from each cluster are randomly selected to participate in the training process. Furthermore, we set $E = 5$ \cite{9759241} and $|\varsigma| = 64$ \cite{10415235}. The initial learning rate is determined using a grid search method \cite{10082940}. The maximum tolerable delay, $T_{\max}$, is 32 hours for CIFAR-10 and 20 hours for MNIST. Additional key wireless parameters are provided in Table \ref{tab:parameter_wireless} \cite{9978924}.

\begin{table}[t]
	\centering
	\caption{Main parameter setting}\label{tab:parameter_wireless}
	\begin{tabular}{|c|c|}	\hline 
		Parameter & Value \\ \hline
		Uplink and downlink frequency, $ \lambda_{k,j} $  and $ \lambda_{k,j} $ & 14 GHz, 12 GHz \\ \hline
		Transmit power, $ P_{\rm{UE}} $ and $ P_{\rm{SAT}} $ & 1 W and 5 W \\ \hline
		Additional loss, $ L^{\rm{AL}} $ & 5 dB \\ \hline
		Noise power spectral density $ n_0 $ & 1.38 $ \times $ 10$ ^{-21} $ \\ \hline
		Radius of satellite antenna &  0.48 m \\ \hline
		Radius of ground terminal antenna & 0.5 m \\ \hline
		Bandwidth of each active client & 5 MHz \\ \hline
	\end{tabular}
\end{table}

\begin{table*}[h]
	\centering
	\caption{Comparison of typical federated learning framework with $ M $ clusters and $ U $ participating clients in one global round}\label{tab:compare_FL}
	\begin{tabular}{|m{3.6cm}<{\centering}|m{2.5cm}<{\centering}|m{3.5cm}<{\centering}|m{4cm}<{\centering}|c|}	\hline 
		Learning Framework	& HFL \cite{10121575}  & PFL \cite{10082940} & Ring Allreduce \cite{9982621} & \textbf{Ours} \\ \hline
		Learning type & Centralized FL & \multicolumn{3}{|c|}{Decentralized FL} \\ \hline
		Extra infrastructure & Ground station & \multicolumn{2}{|c|}{ISL} & None \\ \hline
		Communication traffic (bits)& $ 2\left( {U + M} \right)q $ & $ 2\left( {U + 2M} \right)q $ & $ 2\left[ {U + \left( {M - 1} \right)} \right]q $ & $ 2Uq $ \\ \hline
		Number of established links& $ 2\left( {U{l_{\rm{U2S}}} + M{l_{\rm{S2G}}}} \right) $ & $ \min: 2\left( {U{l_{\rm{U2S}}} + 2M{l_{\rm{ISL}}}} \right) $ $ \max: 2\left( {U{l_{\rm{U2S}}} + 4M{l_{\rm{S2G}}}} \right) $ & $ \min: 2\left[ {U{l_{\rm{U2S}}} + M\left( {M - 1} \right){l_{\rm{ISL}}}} \right] $ $ \max: 2\left[ {U{l_{\rm{U2S}}} + 2M\left( {M - 1} \right){l_{\rm{S2G}}}} \right] $ & $ 2U{l_{\rm{U2S}}}$ \\ \hline
		
	\end{tabular}
\end{table*}

We compare the proposed FedMeld with the following baseline algorithms.

\begin{itemize}
	\item \textbf{Hierarchical FL (HFL)}  \cite{10121575}. 
	In each global round, each serving satellite activates a subset of clients within its coverage area to perform $E$ local iterations before aggregating the models from these clients. After $K+\delta$ global rounds, the satellites transmit the aggregated models to a ground station for global aggregation. Upon re-establishing contact with the ground station, satellites receive the updated global model. In our experiments, two ground stations are used for global aggregation, with parameter sharing facilitated through ground optical fibers.

	\item \textbf{Parallel FL (PFL)} \cite{10082940}. After satellites average the models from the devices within their coverage, they periodically exchange their parameters with neighboring satellites via ISLs, only once during each exchange period. In our experiments, the exchange period is set to $K+\delta$ global rounds.
	
	\item \textbf{Ring Allreduce} \cite{9982621}. 
	Initially, the aggregated model of each satellite is divided into multiple chunks. Each satellite then sends a chunk to the next satellite while simultaneously receiving a chunk from the previous satellite. This iterative process continues until all data chunks are fully aggregated across all serving satellites, resulting in a synchronized model at each satellite. The parameter exchange period is also set to $K+\delta$ global rounds. However, this scheme relies on tightly costly ISLs, which are not fully applied in many existing commercial SGINs.
\end{itemize}

\subsection{Estimation of Hyperparameters and Impact of Estimation Errors}
\subsubsection{Estimation of Hyperparameters}
First, we discuss the hyperparameters $L$ and $\mu$.
Define the function $\phi_j\left ( \mathbf{w},\mathbf{w}' \right ) = \frac{2\left [ F_j\left ( \mathbf{w}  \right ) -F_j\left ( \mathbf{w}'\right ) - \left ( \mathbf{w} -\mathbf{w'} \right ) ^T \triangledown F_j\left ( \mathbf{w}'\right ) \right ] }{\left \|  \mathbf{w} -\mathbf{w'} \right \|^2 }$. 
In principle, the smoothness and strong-convexity constants of each client objective $F_j$ are obtained as follows: $L_j = \sup_{\mathbf{w}, \mathbf{w'}} \phi_j\left ( \mathbf{w},\mathbf{w}' \right )$ and $\mu_j = \inf_{\mathbf{w}, \mathbf{w'}} \phi_j\left ( \mathbf{w},\mathbf{w}' \right )$.
However, directly computing these values is infeasible for deep neural networks due to the extremely high dimensionality and the prohibitive cost of evaluating Hessian eigenvalues.
Following the lightweight empirical approach in~\cite{wei2025optimizing}, we approximate these quantities using local finite-difference tests and cross-client comparisons.

First, on each device $j \in \mathcal{N}$, we perform a finite-difference test on a minibatch of size $| \varsigma_j |$.
Given the model parameter $\mathbf{w}_j$, we apply a small perturbation $\varepsilon$ and compute $\phi_j \left(\mathbf{w}_j+\varepsilon, \mathbf{w}_j \right)$.
Repeating this test for multiple perturbations $\{\varepsilon^{(m)}\}$ yields a set of empirical values $\{\phi_j^{(m)}\}$, where each $\phi_j^{(m)}=\phi_j \left(\mathbf{w}_j+\varepsilon^{(m)}, \mathbf{w}_j \right)$ corresponds to the perturbation $\varepsilon^{(m)}$. The local smoothness estimation is obtained via a minibatch-size weighted average, which is given by $L_{\rm{local}}=\frac{\sum_{j\in \mathcal{N}}  \left ( \max_m \phi_j^{(m)} \right )| \varsigma_j | }{ \sum_{j\in \mathcal{N}} | \varsigma_j |}$. Similarly, the local $\mu$-strongly convex constant is given by $\mu_{\rm{local}} = \frac{\sum_{j\in \mathcal{N}}  \left ( \min_m \phi_j^{(m)} \right )| \varsigma_j | }{ \sum_{j\in \mathcal{N}} | \varsigma_j |}$.

Second, to capture the variance induced by model discrepancy across devices, we further compare client models $\left(\mathbf w_j,\mathbf w_p \right)$ with $j\neq p$ evaluated on the same minibatch and compute $\phi_j\left ( \mathbf{w}_j,\mathbf{w}_p \right )$. This yields the cross-client upper and lower curvature bounds:
$L_{\rm{cross}} = \max_{j \in \mathcal{N}, j \neq p} \phi_j\left ( \mathbf{w}_j,\mathbf{w}_p \right )$ and $ \mu_{\rm{cross}} = \min_{j \in \mathcal{N}, j \neq p} \phi_j\left ( \mathbf{w}_j,\mathbf{w}_p \right )$.
Finally, the global curvature estimates are taken conservatively as $\hat{L} = \max \left \{ L_{\rm{local}}, L_{\rm{cross}} \right \} $ and $\hat{\mu}  = \min \left \{ \mu_{\rm{local}}, \mu_{\rm{cross}} \right \}  $.

Then, we provide the estimation of stochastic gradient variance and magnitude. Similar to \cite{10082940}, we conduct estimation in the first global round.
Each device records the stochastic gradients $\nabla F_j\left({\mathbf{w}}_j ; \varsigma_j\right)$ evaluated on its local minibatch. The empirical gradient variance and magnitude bound are given by $\hat{\sigma}_j^2 = {\rm{Var}}(\nabla F_j\left({\mathbf{w}}_j ; \varsigma_j\right))$ and $\hat{G}= \max_{j \in \mathcal{N}} \left \{ {{{\left\| \nabla {F_j}\left( {{\mathbf{w}}_j;{\varsigma _j}} \right) \right\|}^2}} \right \} $.

\subsubsection{Impact of Estimation Errors}
The empirical quantities $\hat L$, $\hat{\mu}$, $\hat{\sigma}_j^2$, and $\hat{G}$ obtained above are only approximations of the practical values.
Here we analyze how such estimation inaccuracies influence the optimization dynamics and the resulting learning performance.

First, we discuss the relationship between the empirical estimates and the true constants. Since the empirical values $\hat{L}$ and $\hat{\mu}$ are obtained from only a few samples around certain local model points, they may underestimate the global smoothness constant and overestimate the global strong-convexity constant. As a result, it typically holds that $\hat{L} \leq L$ and $\hat{\mu} \geq \mu$.
In addition, the estimation errors in $\hat{\sigma}_j^2$ and $\hat G$ do not exhibit a consistent directional bias, and thus do not systematically enlarge or shrink the optimization region.

Then, we discuss the influence on the optimization variables and learning performance. From (\ref{equ:optimal_delta}), we observe that $\delta^{\ast}$ is independent of all hyperparameters considered in this estimation procedure. Therefore, its value is unaffected by estimation errors. However, the optimal mixing ratio $\alpha^\ast$ depends on the curvature-related constants through the derivative $f'(\alpha)$ of the objective function.
Let $\alpha^\ast$ denote the true optimal mixing ratio and $\hat{\alpha}^\ast$ the value obtained under estimated hyperparameters. Since the estimation biases of $\hat{\sigma}_j^2$ and $\hat G$ are nondirectional, we focus on the effect of the curvature constants $L$ and $\mu$. These constants influence the coefficients $\zeta_1$ and $\zeta_2$, and therefore change the shape of $f'(\alpha)$. When $8L/{\mu} \leq E$ holds, it satisfies
$8\hat{L}/{\hat{\mu}} \leq E$ and
the derivative satisfies $f'\left ( \alpha \mid \hat{L}, \hat{\mu} \right )  < f'\left ( \alpha \mid L, \mu \right ) $ for any $\alpha$. Since $f'(\alpha)$ is non-decreasing related to $\alpha$, the root of $f'\left ( \alpha \right ) $ shifts to the right, implying that $\hat{\alpha}^* > \alpha^*$ when  $f'(\frac{1}{2} \mid \hat{L}, \hat{\mu}) > 0$. If $f'(\frac{1}{2} \mid \hat{L}, \hat{\mu}) \leq 0$, the optimum remains unchanged and $\hat{\alpha}^* = \alpha^*$. On the other hand, when $8\hat{L}/{\hat{\mu}} > E$, the relationship between $\hat{\alpha}^*$ and $\alpha^*$ becomes non-monotonic and cannot be characterized analytically. 
When the estimated hyperparameters lead to $\hat{\alpha}^\ast \neq \alpha^\ast$, the mismatch may lead to a higher upper bound on the training loss.

Despite the estimation errors, the resulting change in the final training loss of FedMeld is modest. As demonstrated in the following experimental results, FedMeld outperforms all benchmark methods, highlighting the robustness of the proposed framework.

\subsection{Communication Costs}
Before evaluating the learning performance, Table \ref{tab:compare_FL} provides a comparison between FedMeld and three baseline algorithms with respect to learning type, model aggregation methods, and communication costs per global round. This analysis considers a scenario with $M$ clusters and $U$ participating clients. In addition, $l_{\rm{U2S}}$, $l_{\rm{S2G}}$, and $l_{\rm{ISL}}$ denote the numbers of user-to-satellite links, satellite-to-ground-station links, and inter-satellite links, respectively.
The $ \min $ and $ \max $ in Table \ref{tab:compare_FL} mainly target the SGI-FL frameworks using ISLs for parameter exchange. That is, PFL and Ring AllReduce. Here, $ \min $ refers to the optimal condition where the satellites serving any two adjacent clusters can directly establish ISLs, thereby minimizing the number of communication links required. In contrast, $ \max $ denotes the worst-case condition where the satellites serving any two adjacent clusters cannot establish ISLs. In this case, if two satellites wish to exchange models, each must first send its model to a ground station and then receive the counterpart model via the ground station. Consequently, a model exchange that would otherwise require one ISL now requires two satellite-to-ground station link transmissions.
We can conclude that the proposed FedMeld is an infrastructure-independent framework with the lowest communication costs.

\begin{figure*}[!h]
\centering\includegraphics[width= 0.8\textwidth]{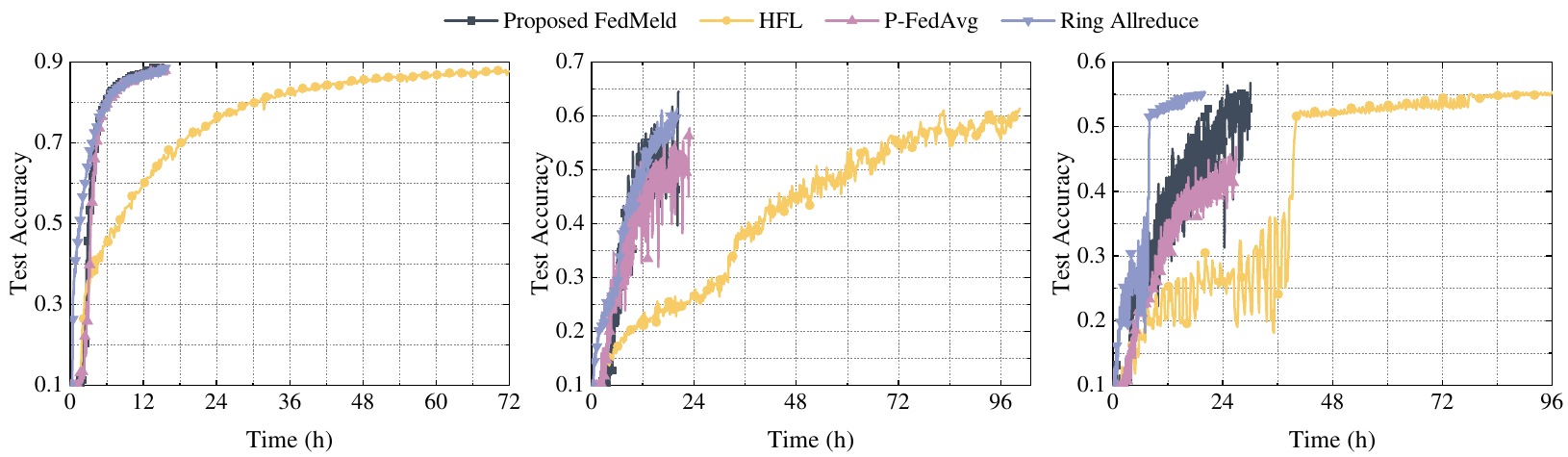}
	\caption{{Test accuracy versus time on CIFAR-10 with IID clients (left), IID clusters with non-IID clients (center), and non-IID clusters with non-IID clients (right).}}
	\label{fig:AccVsDuration_CIFAR}
\end{figure*}

\begin{figure*}[!h]
	\centering
	\includegraphics[width= 0.8\textwidth]{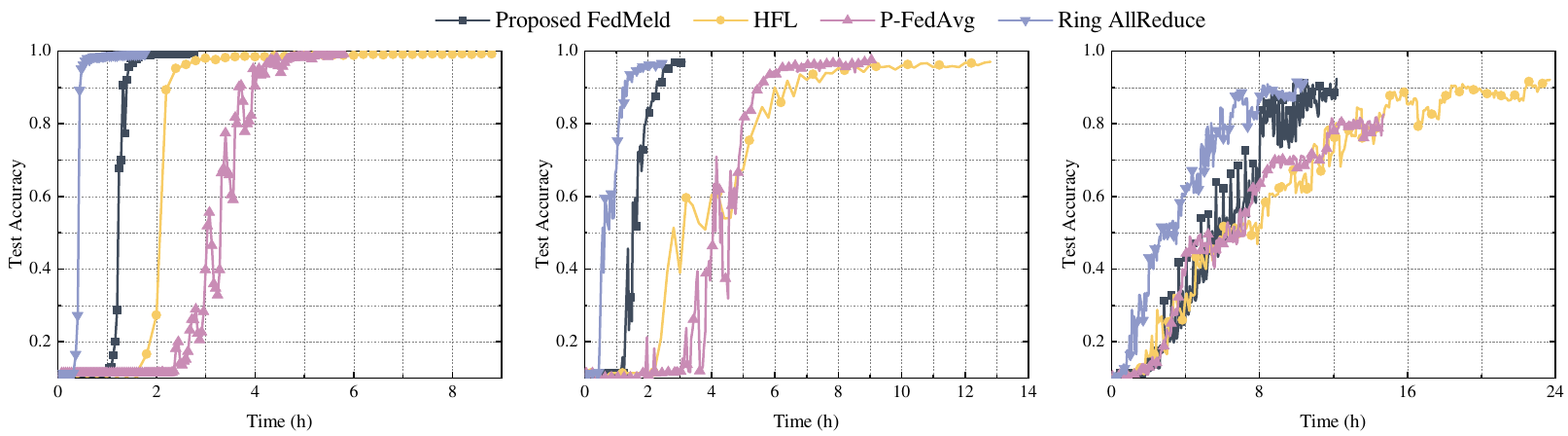}
	\caption{Test accuracy versus time on MNIST in client IID (left), cluster IID with client non-IID (center), and cluster non-IID with client non-IID (right).}
	\label{fig:AccVsDuration_MNIST}
\end{figure*}

\subsection{Learning Performance}
In this part, we evaluate the learning performance of various FL frameworks across two datasets under three distinct data distribution settings. 
Fig. \ref{fig:AccVsDuration_CIFAR} and Fig. \ref{fig:AccVsDuration_MNIST} illustrate the learning performance of the FedMeld compared to three baseline schemes across different data distribution scenarios in each dataset.
These three data distribution settings emulate different levels of task and data heterogeneity across clients and clusters.
\begin{itemize}
	\item \textbf{IID clients}:  Each user receives an equal number of samples randomly drawn from the entire dataset, with images assigned to users following an IID pattern.

	\item \textbf{IID clusters (with non-IID clients)}: Each cluster selects images in proportion to the number of users in that cluster, randomly drawing from all labels. Within each cluster, users are assigned an equal number of images, typically from an average of 3 labels.
	
	\item \textbf{Non-IID clusters (with non-IID clients)}:  Clients within the same cluster receive images drawn from only 3 specific labels, and each client further selects images from 2 labels within the cluster’s data.
\end{itemize}

From the experiment results, the following observations can be made:
1) The proposed FedMeld consistently achieves the highest model accuracy among all frameworks, demonstrating the superior performance of our approach. Compared to HFL, this decentralized FL scheme allows adjacent regions to mix parameters directly and update models more flexibly, avoiding the excessive time required for model aggregation at ground stations.
The advantages of FedMeld over PFL and Ring Allreduce in terms of model accuracy stem from two key aspects: First, the careful design of mixing ratio and global round interval between adjacent regions. Second, FedMeld avoids reliance on ISLs that are not yet fully deployed across the entire constellation.
2) Ring Allreduce exhibits the fastest convergence speed. This is because, after the same number of global rounds, Ring Allreduce completes a full global average, whereas FedMeld and P-FedAvg perform local parameter exchanges between neighboring regions. However, despite its faster convergence, Ring Allreduce introduces significantly higher communication overhead, as it requires frequent utilization of ISLs to mix models across neighboring regions. The reliance on extensive ISL communication can become a bottleneck, particularly in large-scale or bandwidth-constrained satellite networks.
3) The model accuracy of P-FedAvg declines markedly as the non-IID degree of the data increases.  
4) The accuracy of HFL is comparable to that of FedMeld. However, HFL always has the longest training time. This is because the aggregation frequency in HFL is determined by the interval between consecutive passes of all serving satellites over the ground stations.

Next, we will focus on the impact of key parameters on model accuracy and the time cost of model training. Given that the model accuracy under IID clients is quite similar across all frameworks, the subsequent discussion will focus solely on the non-IID cases under varying key parameters.

\subsection{Effects of Satellite Number}
We investigate the effect of the number of satellites on learning performance across different frameworks, as shown in Fig. \ref{fig:SatNumber_Cifar} and Fig. \ref{fig:SatNumber_MNIST}. The number of satellites directly influences the session duration between space and ground, affecting learning performance by controlling $K$.
As satellites become denser, the model accuracy of FedMeld initially improves and then declines in Fig. \ref{fig:satnum_cifar_clusteriid}, while accuracy increases monotonically in other settings. Compared to other methods, FedMeld generally demonstrates superior model accuracy, except in the case of non-IID clusters in Fig. \ref{fig:acc_satnum_mnist} with 88 satellites per orbit.
Regarding training time, it shows a monotonically increasing trend due to the shorter serving time of each satellite. Additionally, dense satellite networks lead to more frequent handovers and increased launch costs. Thus, beyond the trade-off between model accuracy and training time, there is also a trade-off between learning performance and handover/engineering costs.

\begin{figure}[h]
	\centering
	\subfigure[{Test accuracy of IID clusters (left) and non-IID clusters (right)}]{\includegraphics[width =0.47\textwidth]{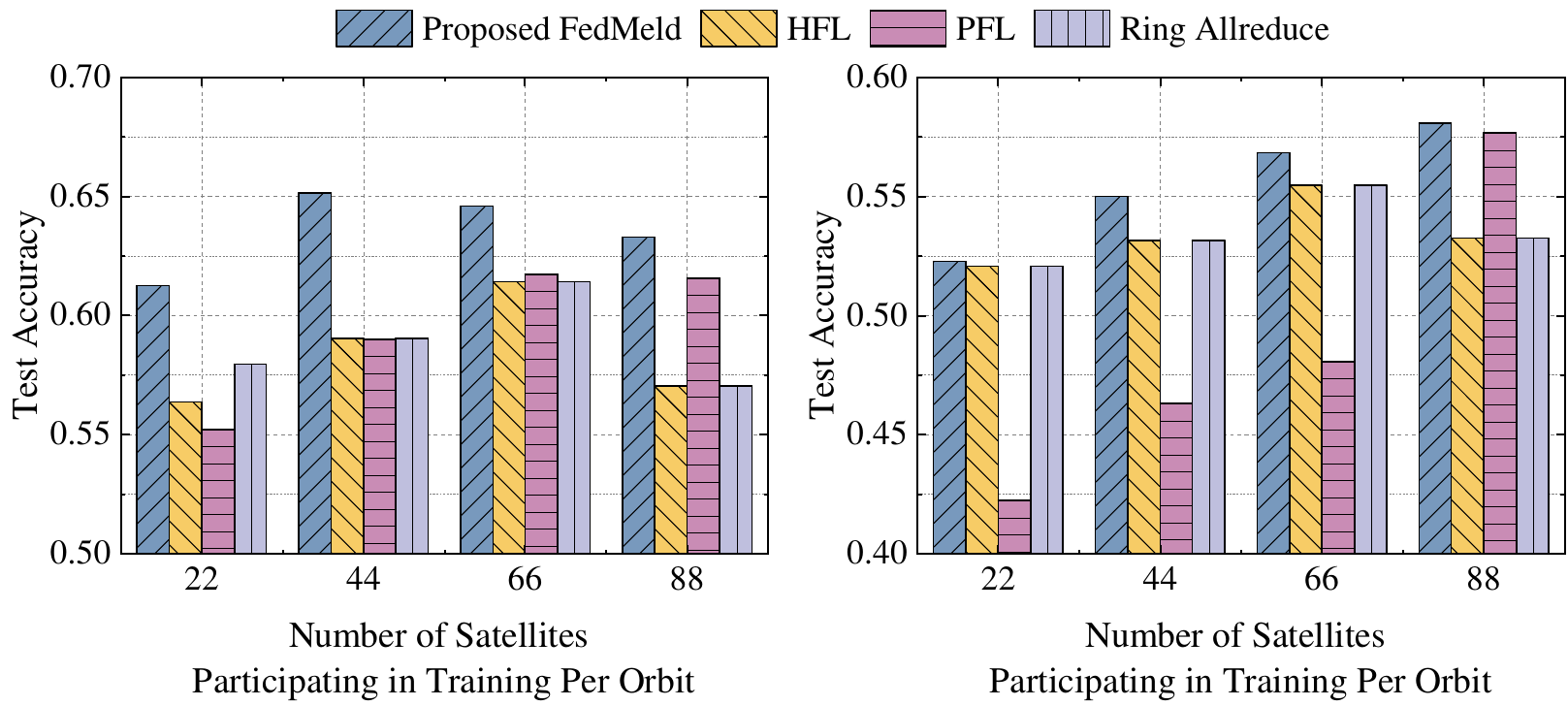}\label{fig:satnum_cifar_clusteriid}}
	\subfigure[{Training time of IID clusters (left) and non-IID clusters (right)}]{\includegraphics[width =0.47\textwidth]{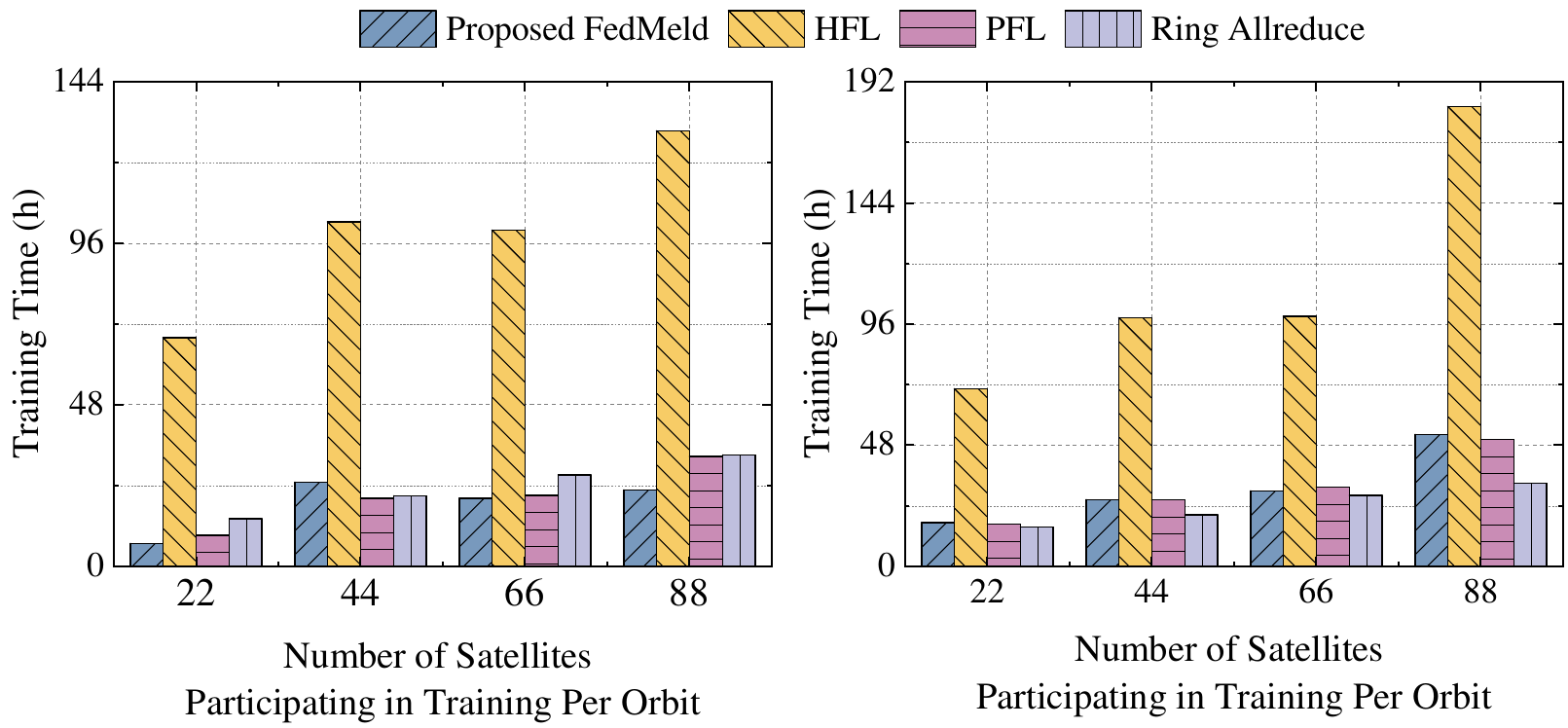}}
	\caption{{Effect of number of satellites participating in training per orbit on test accuracy and training time on CIFAR-10.}} 
	\label{fig:SatNumber_Cifar}  
\end{figure}

\begin{figure}[h]
	\centering
	\subfigure[{Test accuracy of IID clusters (left) and non-IID clusters (right)}]{\includegraphics[width =0.47\textwidth]{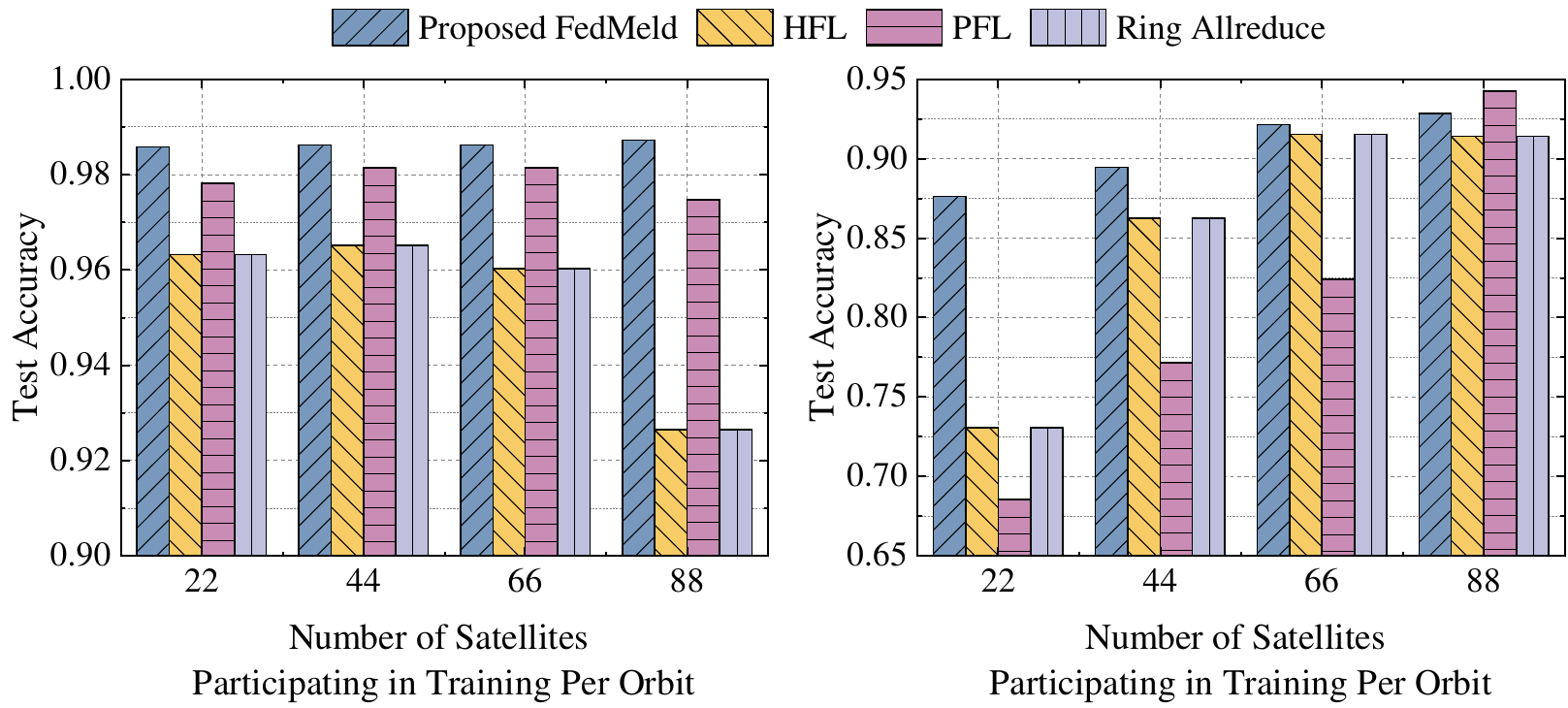}\label{fig:acc_satnum_mnist}}
	\subfigure[{Training time of IID clusters (left) and non-IID clusters (right)}]{\includegraphics[width =0.47\textwidth]{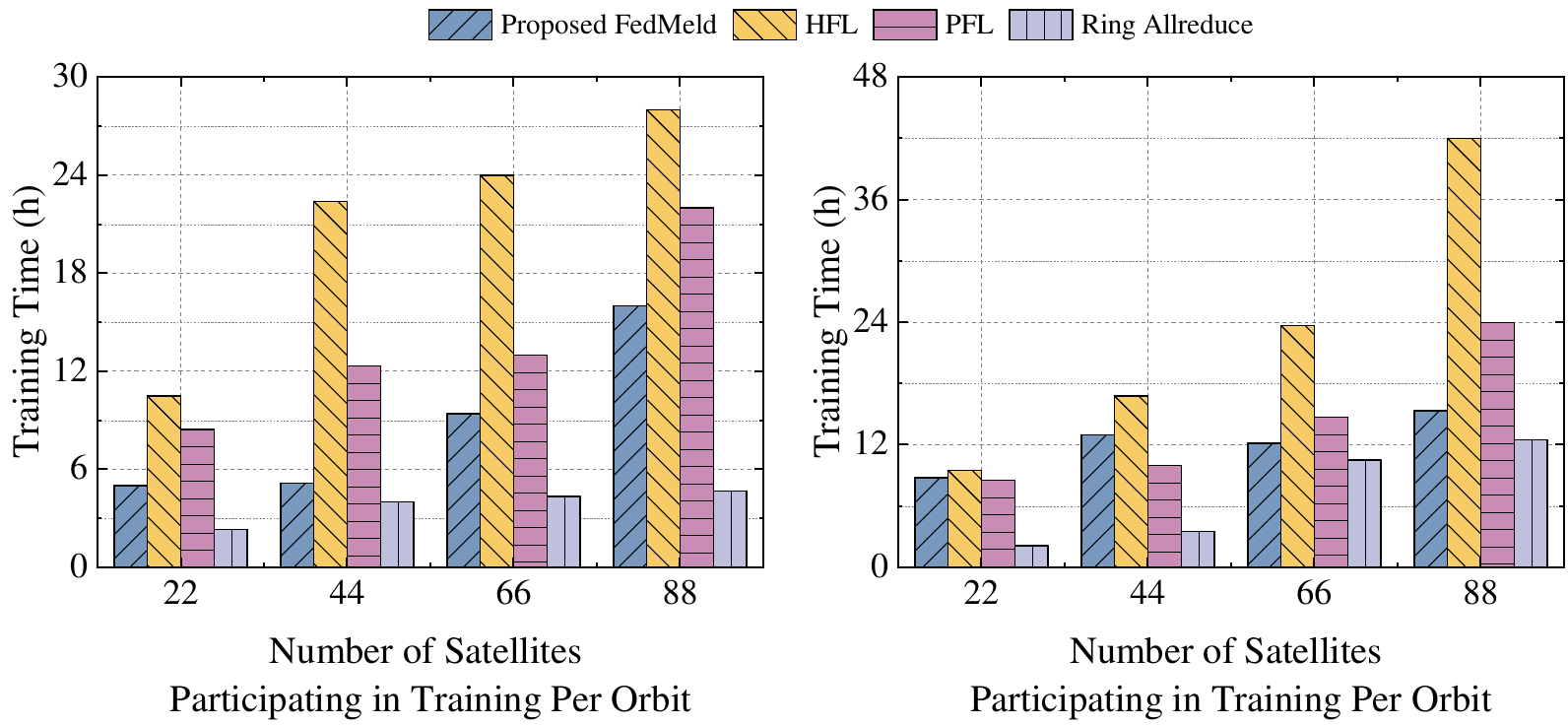}}
	\caption{{Effect of satellite number on test accuracy and training time on MNIST.}} 
	\label{fig:SatNumber_MNIST}  
\end{figure}

\subsection{Effects of Latency Constraint}
In Fig. \ref{fig:MaximalDelay_Cifar} and Fig. \ref{fig:MaximalDelay_MNIST}, we analyze the impact of maximal tolerable delay $T_{\max}$ on the performance of FedMeld. As the latency constraint becomes more relaxed, it allows more fresh parameter mixing when the serving satellite flies over the next cluster, i.e., $\delta^{\ast}$ is smaller. We observe that when $T_{\max}$ increases, model accuracy improves while convergence speed slows. These observations confirm the effectiveness of the proof that the objective function $f\left(\delta, \alpha\right)$ of $\mathcal{P}1$ is monotonically increasing with respect to $\delta$ in Lemma \ref{lemma:subsequent}.

\begin{figure}[h]
	\centering
	\subfigure[IID clusters]{\includegraphics[width =0.24\textwidth]{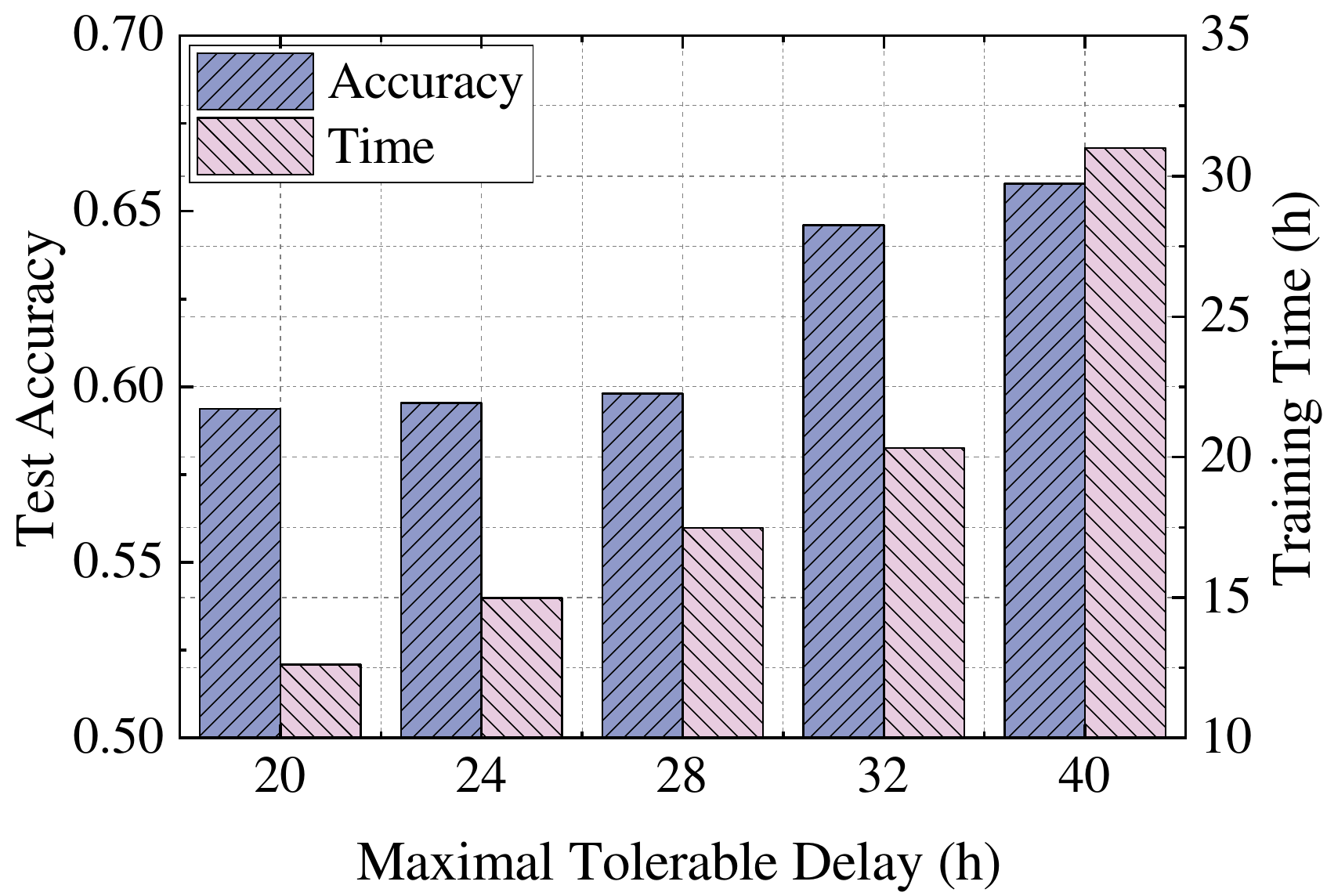}}
	\subfigure[Non-IID clusters]{\includegraphics[width =0.24\textwidth]{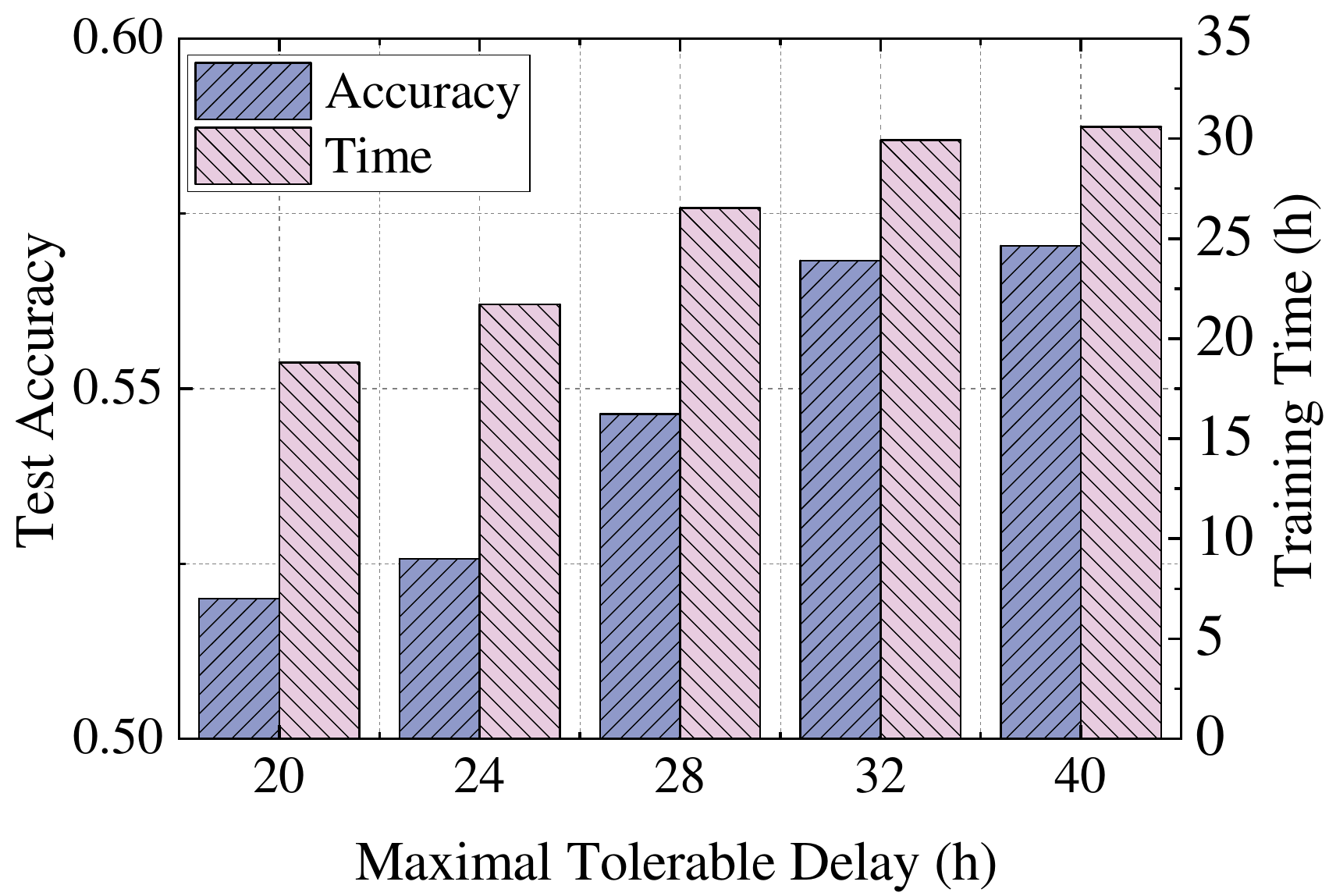}}
	\caption{Effect of maximal tolerable delay with different non-IID degrees on CIFAR-10.} 
	\label{fig:MaximalDelay_Cifar}  
\end{figure}

\begin{figure}[!h]
	\centering
	\subfigure[IID clusters]{\includegraphics[width =0.24\textwidth]{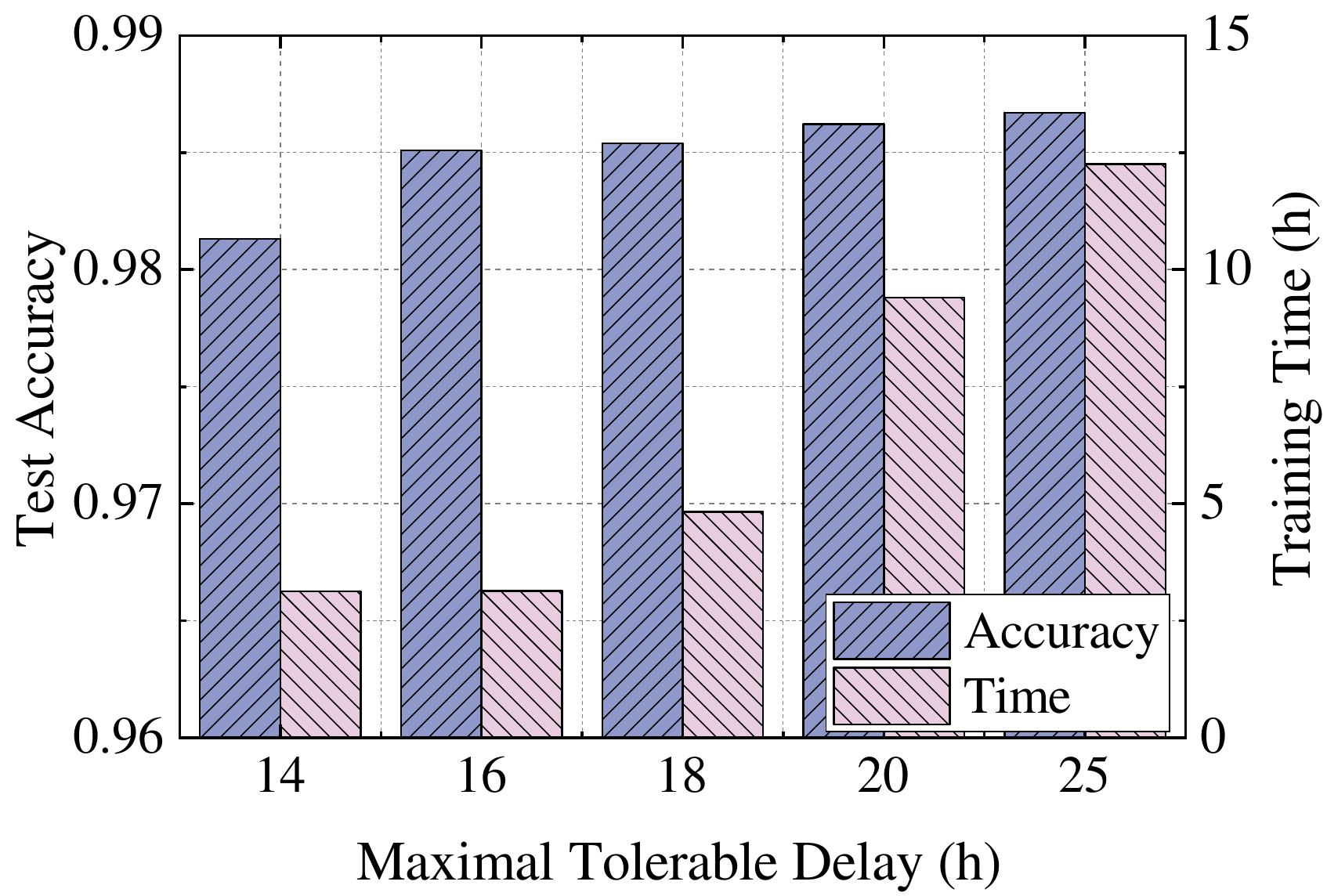}}
	\subfigure[Non-IID clusters]{\includegraphics[width =0.24\textwidth]{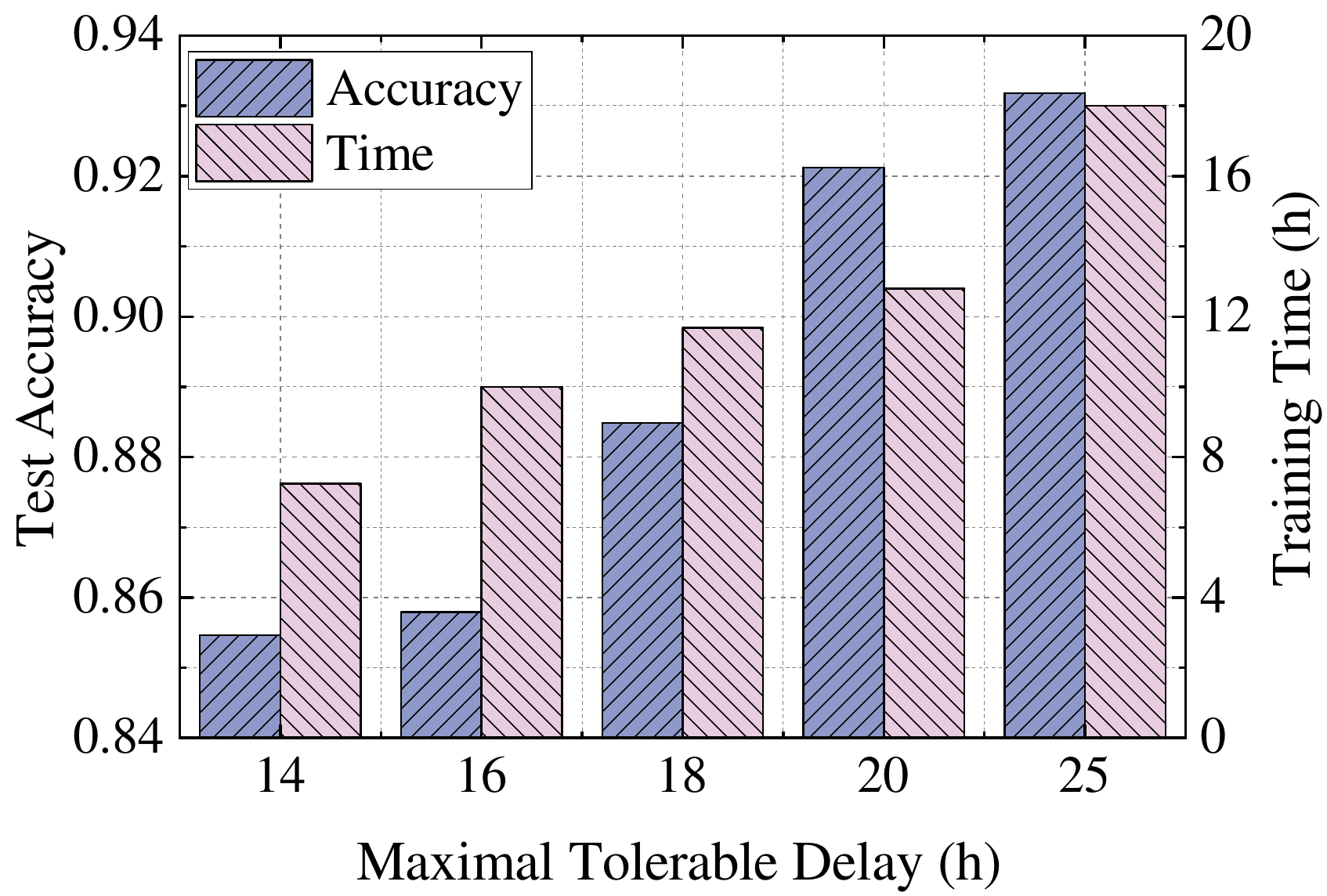}}
	\caption{Effect of maximal tolerable delay with different non-IID degrees on MNIST.} 
	\label{fig:MaximalDelay_MNIST}
\end{figure}

\subsection{Effect of Data Heterogeneity}

Fig. \ref{fig:gamma_Cifar} and Fig. \ref{fig:gamma_MNIST} illustrate the impact of intra-cluster and inter-cluster data heterogeneity on the test accuracy and optimal mixing ratio for CIFAR-10 and MNIST, respectively.
In these experiments, the degree of data heterogeneity is controlled by the Dirichlet concentration parameters, which determine the skewness of the data distributions across clusters and among users within each cluster~\cite{Li2020n,chen2024heterogeneity}.
The concentration parameters are varied in $\left \{  0.5, 1, 2, 5, 10\right \} $, covering a range from highly non-IID to nearly IID distributions.
A smaller concentration parameter corresponds to a more imbalanced (highly non-IID) data distribution, whereas a larger parameter leads to a more uniform (approaching IID) allocation of samples.
As shown in the figures, when the heterogeneity increases, i.e., the Dirichlet concentration parameter decreases, the optimal mixing ratio consistently decreases, confirming our Remark \ref{remark:optimal_mixing_ratio} that greater distribution disparity requires weaker inter-cluster or inter-user model mixing to mitigate negative transfer across dissimilar datasets.
Meanwhile, a lower degree of heterogeneity generally yields higher test accuracy, as more balanced data partitions enable stable aggregation and faster convergence.
These results collectively demonstrate that the proposed adaptive mixing mechanism can dynamically adjust to different levels of data heterogeneity.

\begin{figure}[t]
	\centering
	\subfigure[IID clusters]{\includegraphics[width =0.24\textwidth]{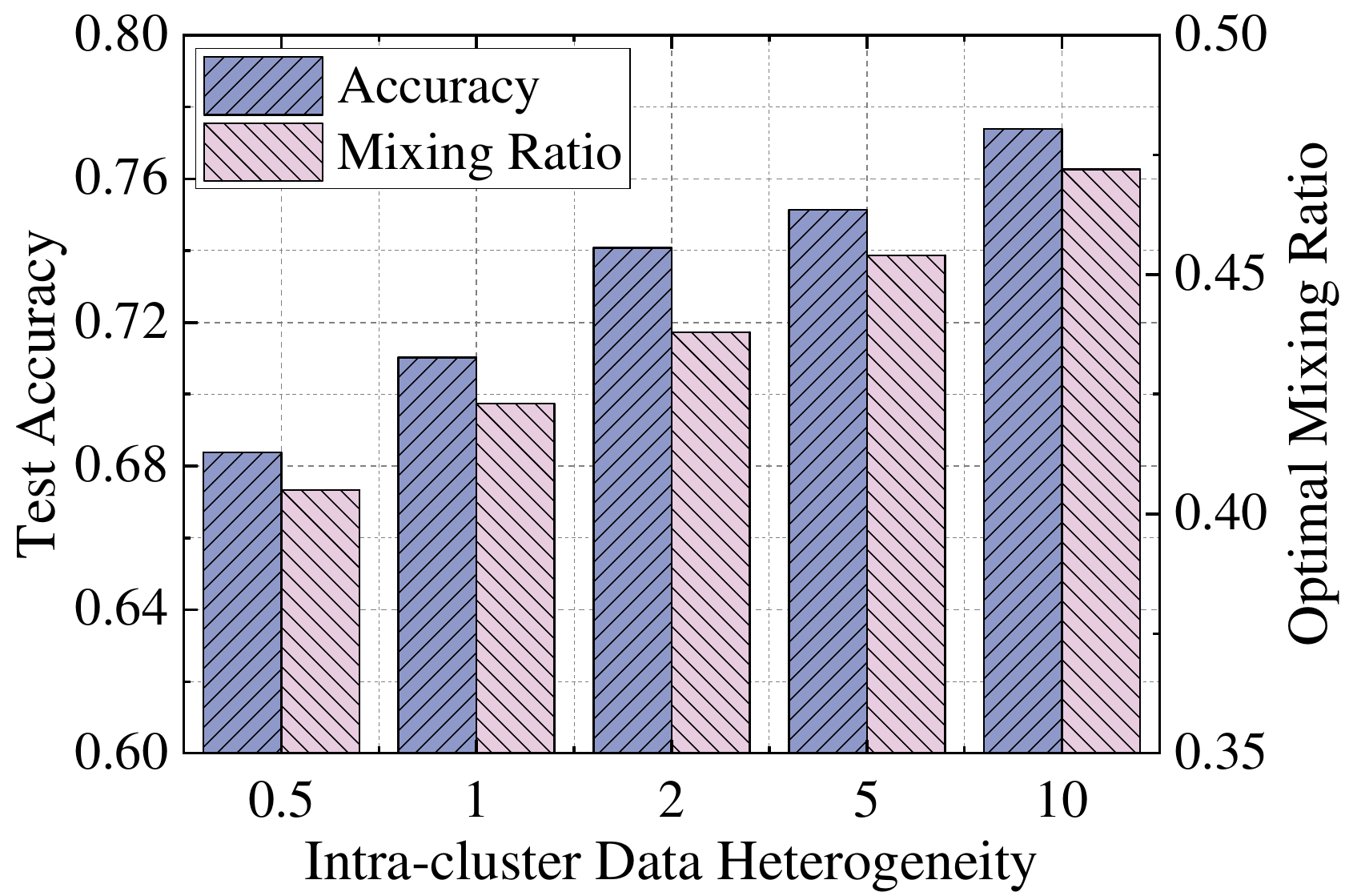}}
	\subfigure[Non-IID clusters]{\includegraphics[width =0.24\textwidth]{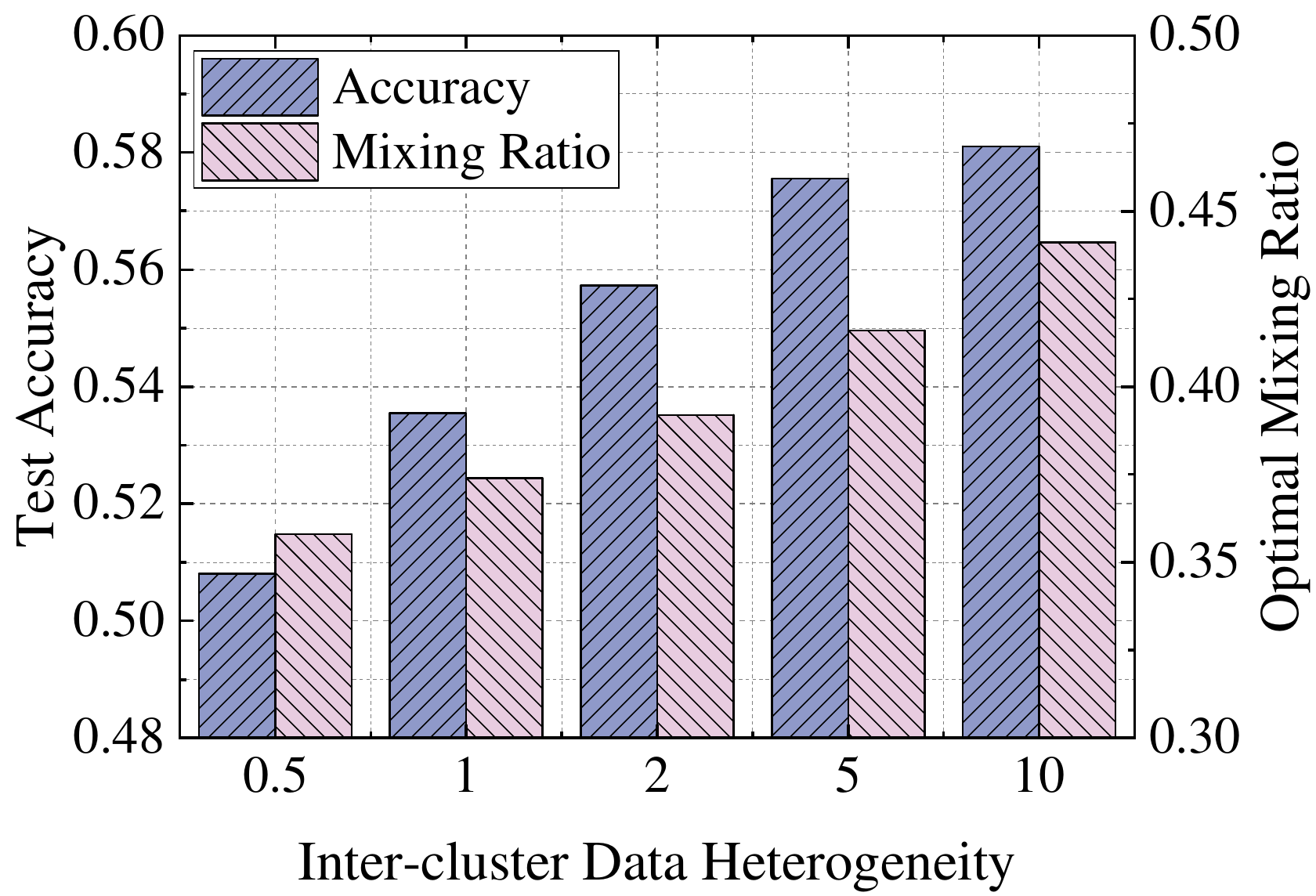}}
	\caption{{Effect of data heterogeneity with different non-IID degrees on CIFAR-10.}} 
	\label{fig:gamma_Cifar}  
\end{figure}

\begin{figure}[t]
	\centering
	\subfigure[IID clusters]{\includegraphics[width =0.24\textwidth]{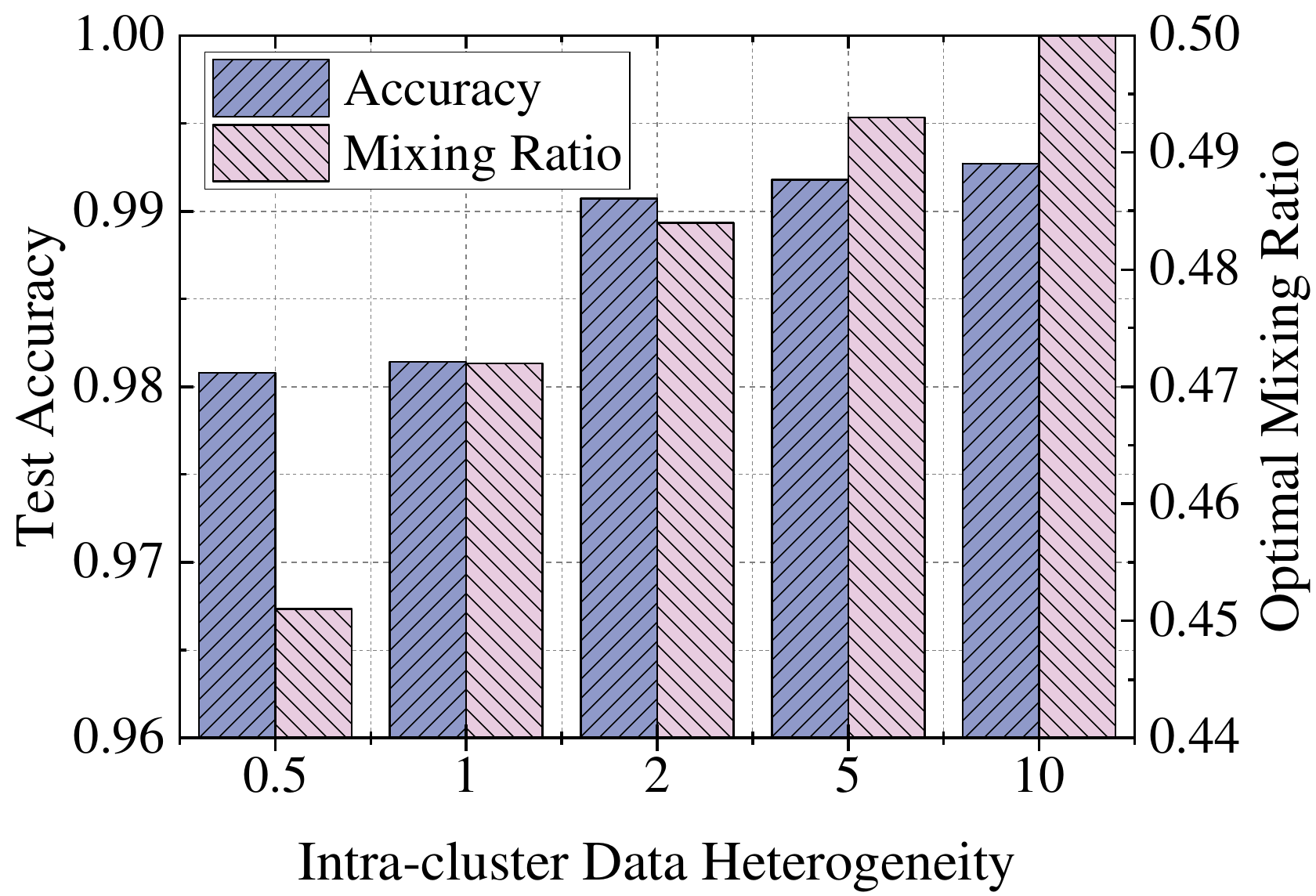}}
	\subfigure[Non-IID clusters]{\includegraphics[width =0.24\textwidth]{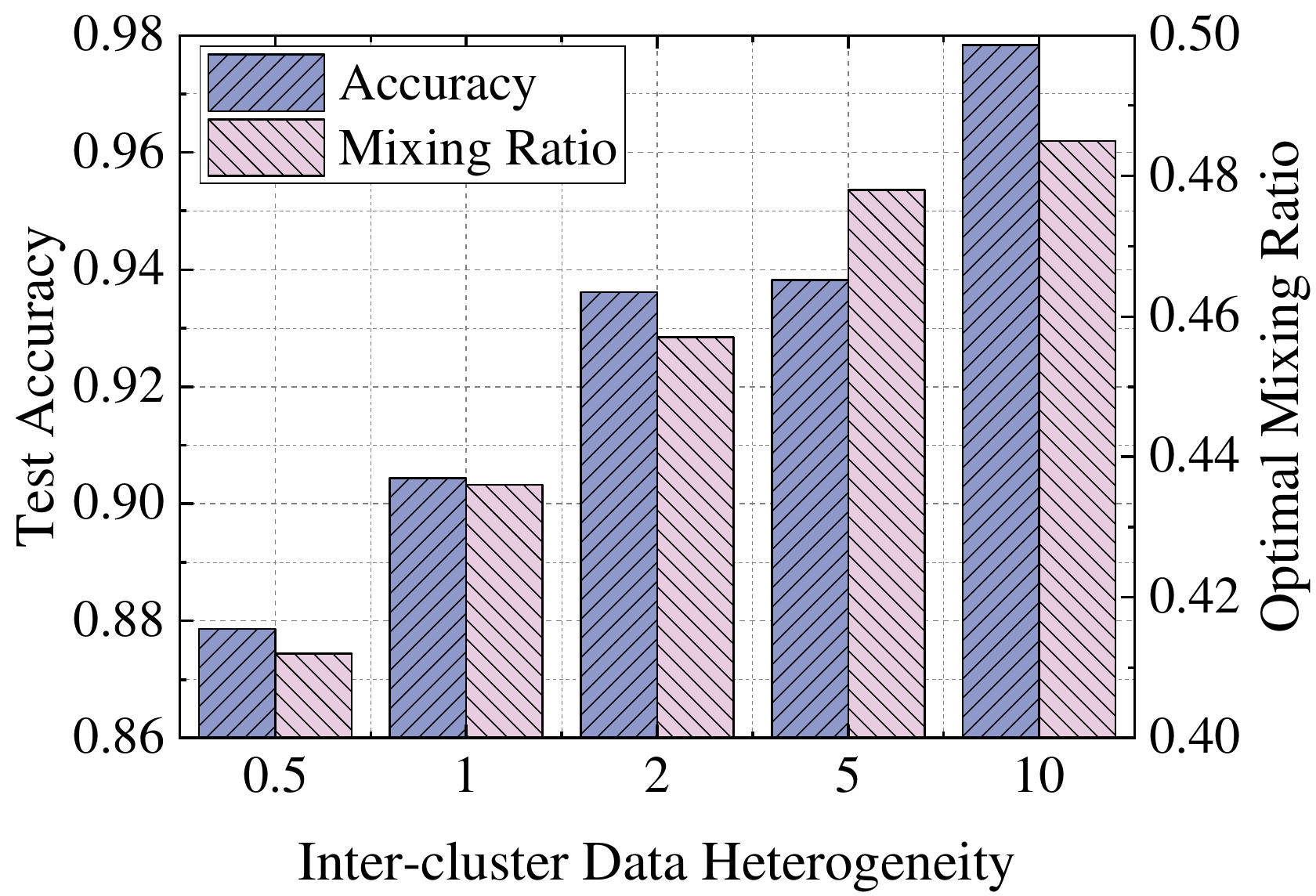}}
	\caption{{Effect of data heterogeneity with different non-IID degrees on MNIST.}} 
	\label{fig:gamma_MNIST}  
\end{figure}

\section{Conclusion}\label{sec:conclusion}
In this paper, we propose the novel FedMeld framework for SGINs. By leveraging the predictable movement patterns and SCF capabilities of satellites, the infrastructure-free FedMeld enables parameter aggregation across different terrestrial regions without relying on ground stations or ISLs. Our findings indicate that the optimal global round interval of model mixing between adjacent areas is highly influenced by latency constraints and handover frequency, while the optimal mixing ratio of historical models from adjacent regions is determined by the degree of non-IID data distribution.

This work pioneers the concept of an infrastructure-free FL framework within the context of SGI-FL. The current FedMeld algorithm can be extended to more complex scenarios, such as implementing region-specific round intervals and adaptive mixing ratios. Practical challenges, such as balancing learning performance with handover and launch costs, merit further exploration. Additionally, investigating efficient inference strategies and model downloading mechanisms for SGINs represents a promising avenue for future research.

\ifCLASSOPTIONcaptionsoff
\newpage
\fi

\begin{appendices}

\section{Proof of Theorem \ref{theorem:convergence_bound}} \label{proof:theorem_convergence_bound}

From Assumption \ref{ass:unbiased_and_variance},  the upper bound of $ \mathbb{E}{\left\| {{{\mathbf{g}}_t} - {{\overline {\mathbf{g}}}_t}} \right\|^2} $ follows that
\begin{equation}\label{equ:square_norm_gradient}
	\begin{split}
		&	\mathbb{E}{\left\| {{{\mathbf{g}}_t} - {{\overline {\mathbf{g}}}_t}} \right\|^2} \\
		&	= {\left\| {\sum\limits_{i \in {\mathcal{M}}} {\sum\limits_{j \in {{\mathcal{N}}_i}} {\frac{1}{{M{N_i}}}\left[ {\nabla {F_j}\left( {{\mathbf{w}}_{t,j}^s,\varsigma _{t,j}^s} \right) - \nabla {F_j}\left( {{{\mathbf{w}}_{t,j}}} \right)} \right]} } } \right\|^2}  \\
		&	\leq \sum\limits_{i \in {\mathcal{M}}} {\sum\limits_{j \in {{\mathcal{N}}_i}} {\frac{{\sigma _j^2}}{{{{\left( {M{N_i}} \right)}^2}}}} }.  \\ 
	\end{split}
\end{equation}

Let $ {\Delta _t} = \mathbb{E}{\left\| {{\overline{\mathbf{w}}_t} - {{\mathbf{w}}^{ \star }}} \right\|^2} $.
By substituting (\ref{equ:square_norm_gradient}) into (\ref{equ:one_step_SGD}), we have
\begin{equation}\label{equ:upper_vw}
	{\left\| {{{\overline {\mathbf{v}}}_{t + 1}} - {{\mathbf{w}}^{ \star }}} \right\|^2}	\leq \left( {1 - \mu {\eta _t}} \right){\Delta _t} + \eta _t^2B +2Q_t,
\end{equation}
where $ B = \sum\limits_{i \in {\mathcal{M}}} {\sum\limits_{j \in {{\mathcal{N}}_i}} {\frac{{\sigma _j^2}}{{{{\left( {M{N_i}} \right)}^2}}}} }  + 6L\Gamma  $ and $ {Q_t} = \mathbb{E}\sum\limits_{i \in {\mathcal{M}}} {\sum\limits_{j \in {{\mathcal{N}}_i}} {\frac{1}{{M{N_i}}}} } {\left\| {{{\overline {\mathbf{w}} }_t} - {{\mathbf{w}}_{t,j}}} \right\|^2} $.

First, we discuss the case that $ A_{t+1} \neq 0 $.
From (\ref{equ:update_rule_w}), it is clear that when $ A_{t+1} \neq 0 $, $ {{\overline {\mathbf{w}}}_{t+1}} = {{\overline {\mathbf{v}}}_{t+1}} $ holds, and we have 
\begin{equation}\label{equ:delta_nohandover}
	{\Delta _{t + 1}} \leq \left( {1 - {\eta _t}\mu } \right){\Delta _t} + \eta _t^2B + 2{Q_t}.
\end{equation}
Thus, we only need to verify $ {\Delta _{t + 1}} \leq \left( {1 - {\eta _t}\mu } \right){\Delta _t} + \eta _t^2B $, and then (\ref{equ:delta_nohandover}) always holds.

Choose a diminishing stepsize learning rate $ \eta_t = \frac{\beta}{t + \gamma}  $ for some $ \mu \beta > 1 $ and $ \gamma > 0 $ such that $ \eta_1 \leq \min \left\{\frac{1}{\mu}, \frac{1}{4L} \right\} =\frac{1}{4L} $ and $ \eta_t \leq 2\eta_{t+E} $. By setting $  v = \max \left\{ \frac{\beta^2 B}{\beta \mu - 1}, (\gamma + 1) \Delta_1 \right\}
$, then we will prove $ {\Delta _t} \leq \left( {\frac{1}{{1 - \alpha }} - \alpha + \frac{1}{2} } \right)\frac{v}{{t + \gamma }} $ for $ A_{t+1} \neq 0 $. Assume that this conclusion holds for some $ t $, then we have
\begin{align}
	&	{\Delta _{t + 1}} \leq \left( {1 - {\eta _t}\mu } \right){\Delta _t} + \eta _t^2B  \nonumber \\
	&	\leq \left( {1 - \frac{{\mu \beta }}{{t + \gamma }}} \right)\left( {\frac{1}{{1 - \alpha }} - \alpha + \frac{1}{2}} \right)\frac{v}{{t + \gamma }} + \frac{{B{\beta ^2}}}{{{{\left( {t + \gamma } \right)}^2}}} \nonumber \\
	&	= \frac{{t + \gamma  - 1}}{{{{\left( {t + \gamma } \right)}^2}}}\left( {\frac{1}{{1 - \alpha }} - \alpha+ \frac{1}{2} } \right)v \nonumber \\
	& + \left[ {\frac{{B{\beta ^2}}}{{{{\left( {t + \gamma } \right)}^2}}} - \frac{{\mu \beta  - 1}}{{{{\left( {t + \gamma } \right)}^2}}}\left( {\frac{1}{{1 - \alpha }} - \alpha + \frac{1}{2} } \right)v} \right] \nonumber \\
	\label{equ:expression_larger_one}	&	\leq \left( {\frac{1}{{1 - \alpha }} - \alpha + \frac{1}{2} } \right)\frac{{t + \gamma  - 1}}{{{{\left( {t + \gamma } \right)}^2}}}v  \\
	&	\leq \left( {\frac{1}{{1 - \alpha }} - \alpha + \frac{1}{2} } \right)\frac{1}{{t + \gamma  + 1}}v, \nonumber
\end{align}
where (\ref{equ:expression_larger_one}) follows from $ B{\beta ^2} \leq \left(\mu \beta  - 1 \right)v  $ and $ {\frac{1}{{1 - \alpha }} - \alpha + \frac{1}{2} } > 1 $.

Then, we discuss the case that $ A_{t+1} = 1 $.
When $ A_{t+1} = 1 $, we have $ {\overline {\mathbf{w}}_{t+1}} = \left( {1 - \alpha } \right){\overline {\mathbf{v}}_{t+1}} + \alpha {\overline {\mathbf{w}}_{{t+1} - \delta E }} $. Then, the lower bounded of $ \mathbb{E}{\left\| {{{\overline {\mathbf{v}}}_{t + 1}} - {{\mathbf{w}}^{ \star }}} \right\|^2} $ can be obtained as follows.
\begin{align}
	&	\mathbb{E}{\left\| {{{\overline {\mathbf{v}}}_{t + 1}} - {{\mathbf{w}}^{ \star }}} \right\|^2} \nonumber \\
	&	= \mathbb{E}{\left\| {\frac{{{{\overline {\mathbf{w}}}_{t + 1}} - \alpha {{\overline {\mathbf{w}}}_{t + 1 - \delta E }}}}{{1 - \alpha }} - {{\mathbf{w}}^{ \star }}} \right\|^2} \nonumber \\
	&	= \mathbb{E}{\left\| {\frac{1}{{1 - \alpha }}\left( {{{\overline {\mathbf{w}}}_{t + 1}} - {{\mathbf{w}}^{ \star }}} \right) - \frac{\alpha }{{1 - \alpha }}\left( {{{\overline {\mathbf{w}}}_{t + 1 - \delta E }} - {{\mathbf{w}}^{ \star }}} \right)} \right\|^2}  \nonumber \\ 
	&  = \frac{1}{{{{\left( {1 - \alpha } \right)}^2}}}{\left\| {{{\overline {\mathbf{w}}}_{t + 1}} - {{\mathbf{w}}^{ \star }}} \right\|^2} + \frac{{{\alpha ^2}}}{{{{\left( {1 - \alpha } \right)}^2}}}{\left\| {{{\overline {\mathbf{w}}}_{t + 1 - \delta E }} - {{\mathbf{w}}^{ \star }}} \right\|^2} \nonumber \\
	&- \frac{{2\alpha }}{{{{\left( {1 - \alpha } \right)}^2}}}\left\langle {{{\overline {\mathbf{w}}}_{t + 1}} - {{\mathbf{w}}^{ \star }},{{\overline {\mathbf{w}}}_{t + 1 - \delta E }} - {{\mathbf{w}}^{ \star }}} \right\rangle \nonumber \\
	& \geq \frac{1}{{1 - \alpha }}{\left\| {{{\overline {\mathbf{w}}}_{t + 1}} - {{\mathbf{w}}^{ \star }}} \right\|^2} - \frac{\alpha }{{1 - \alpha }}{\left\| {{{\overline {\mathbf{w}}}_{t + 1 - \delta E}} - {{\mathbf{w}}^{ \star }}} \right\|^2}, \label{equ:lower_bound}
\end{align}
The last inequality comes from Cauchy-Schwarz Inequality and Young's inequality: 
\begin{equation*}
	\begin{split}
		&	\left\langle {{{\overline {\mathbf{w}}}_{t + 1}} - {{\mathbf{w}}^{ \star }},{{\overline {\mathbf{w}}}_{t + 1 - \delta E }} - {{\mathbf{w}}^{ \star }}} \right\rangle  \\
		& \leq \left\| {{{\overline {\mathbf{w}}}_{t + 1}} - {{\mathbf{w}}^{ \star }}} \right\|\left\| {{{\overline {\mathbf{w}}}_{t + 1 - \delta E }} - {{\mathbf{w}}^{ \star }}} \right\| \\
		& \leq \frac{1}{2}\left( {{{\left\| {{{\overline {\mathbf{w}}}_{t + 1}} - {{\mathbf{w}}^{ \star }}} \right\|}^2} + {{\left\| {{{\overline {\mathbf{w}}}_{t + 1 - \delta E}} - {{\mathbf{w}}^{ \star }}} \right\|}^2}} \right).
	\end{split}
\end{equation*}
By substituting (\ref{equ:lower_bound}) into (\ref{equ:upper_vw}), we have 
\begin{equation}\label{equ:delta_t+1}
	{\Delta _{t + 1}} \leq \left( {1 - \alpha } \right)\left[ {\left( {1 - {\eta _t}\mu } \right){\Delta _t} + \eta _t^2B} \right] + \alpha {\Delta _{t + 1 - \delta E}} + 2{Q_t},
\end{equation}
Similarly, we only need to verify $ {\Delta _{t + 1}} \leq \left( {1 - \alpha } \right)\left[ {\left( {1 - {\eta _t}\mu } \right){\Delta _t} + \eta _t^2B} \right] + \alpha {\Delta _{t + 1 - \delta E}} $, and then (\ref{equ:delta_t+1}) always holds.
We also prove $ {\Delta _t} \le \left( {\frac{1}{{1 - \alpha }} - \alpha  + \frac{1}{2}} \right)\frac{v}{{t + \gamma }} $ by induction. Assume that this conclusion holds for some $ t $, then we have
\begin{equation}\label{equ:upper_new}
\resizebox{1\hsize}{!}{$	\begin{split}
		&	{\Delta _{t + 1}} \leq \left( {1 - \alpha } \right)\left( {1 - {\eta _t}\mu } \right){\Delta _t} + \left( {1 - \alpha } \right)\eta _t^2B + \alpha {\Delta _{t + 1 - \delta E}}  \\
		&	\leq \left( {1 - \frac{{\mu \beta }}{{t + \gamma }}} \right)\frac{{\left( {2{\alpha ^2} - 3\alpha  + 3} \right)v}}{{2\left( {t + \gamma } \right)}} + \frac{{\left( {1 - \alpha } \right)B{\beta ^2}}}{{{{\left( {t + \gamma } \right)}^2}}} \\
		&	+ \frac{{\alpha \left( {2{\alpha ^2} - 3\alpha  + 3} \right)}}{{2\left( {1 - \alpha } \right)}}\frac{v}{{t + 1 - \delta E + \gamma }}  \\
		& = \frac{{\left( {2{\alpha ^2} - 3\alpha  + 3} \right)\left( {t + \gamma  - 1} \right)v}}{{2{{\left( {t + \gamma } \right)}^2}}} + \frac{{\alpha \left( {2{\alpha ^2} - 3\alpha  + 3} \right)v}}{{2\left( {1 - \alpha } \right)\left( {t + 1 - \delta E + \gamma } \right)}}\\
		&  - \frac{{\left( {2{\alpha ^2} - \alpha  + 1} \right)\left( {\mu \beta  - 1} \right)v}}{{2{{\left( {t + \gamma } \right)}^2}}}  \\
		& + \left( {1 - \alpha } \right)\left[ {\frac{{B{\beta ^2}}}{{{{\left( {t + \gamma } \right)}^2}}} - \frac{{\mu \beta  - 1}}{{{{\left( {t + \gamma } \right)}^2}}}v} \right]\\
		& \leq  \frac{{\left( {2{\alpha ^2} - 3\alpha  + 3} \right)\left( {t + \gamma  - 1} \right)v}}{{2{{\left( {t + \gamma } \right)}^2}}} + \frac{{\alpha \left( {2{\alpha ^2} - 3\alpha  + 3} \right)v}}{{2\left( {1 - \alpha } \right)\left( {t + 1 - \delta E + \gamma } \right)}} \\
		&  - \frac{{\left( {2{\alpha ^2} - \alpha  + 1} \right)\left( {\mu \beta  - 1} \right)v}}{{2{{\left( {t + \gamma } \right)}^2}}} \\
		&  \le \left( {\frac{1}{{1 - \alpha }} - \alpha  + \frac{1}{2}} \right)\frac{v}{{t + \gamma  + 1}}.
	\end{split} $}
\end{equation}
When $ \delta E \le \frac{{m\left( \alpha  \right){{\left( {t + \gamma  + 1} \right)}^2}}}{{{{\left( {t + \gamma } \right)}^2} + m\left( \alpha  \right)\left( {t + \gamma  + 1} \right)}} $, the last inequality of (\ref{equ:upper_new}) holds.

By setting $ \beta = \frac{2}{\mu} $ and $ \gamma = \max \left\{ \frac{8L}{\mu}, E  \right\} - 1 $, then $ \eta_t $ satisfies $ \eta_t  \leq 2 \eta_{t+E}$ for any step $ t $. Then, we have  
\begin{align*}
	v = \max \left\{ {\frac{{{\beta ^2}B}}{{\beta \mu  - 1}},\left( {\gamma  + 1} \right){\Delta _1}} \right\} & \leq \frac{{{\beta ^2}B}}{{\beta \mu  - 1}} + \left( {\gamma  + 1} \right){\Delta _1} \\
	& \leq \frac{{4B}}{{{ \mu ^2}}} + \left( {\gamma  + 1} \right){\Delta _1}.
\end{align*}
According to Assumption \ref{ass:smooth}, the convergence upper bound can be given by
\begin{equation}\label{equ:gap_barw_wstar}
\resizebox{1\hsize}{!}{$	\begin{split}
		&	\mathbb{E}\left[ {F\left( {{{\overline {\mathbf{w}}}_t}} \right)} \right] - F\left( {{{\mathbf{w}}^{ \star }}} \right)  \leq \frac{L}{2}{\Delta _t} \\
		&	\leq \frac{L}{2}\left[ \left( {\frac{1}{{1 - \alpha }} - \alpha + \frac{1}{2} } \right){\frac{{v}}{{t + \gamma }} + {2Q_{t - 1}}} \right] \\
		& \leq \left( {\frac{1}{{1 - \alpha }} - \alpha  + \frac{1}{2} } \right)
		\frac{{L}}{{ \mu \left( {t + \gamma  } \right)}} \left[ {\frac{{2B}}{{{ \mu}}} + \frac{{\mu \left( {\gamma  + 1} \right)}}{2}{\Delta _1}} \right] + L{Q_{t - 1}}.
	\end{split} $}
\end{equation}
This completes the proof.

\section{Proof of Lemma \ref{lemma:relationship_Qt}} \label{proof:lemma_relationship_Qt}
We introduce an auxiliary variable $ \tilde{t} $. When $ t \notin {\mathcal{I}}_E $, then  $ \tilde t = \left\lfloor {t/E} \right\rfloor E $. When $ t \in {\mathcal{I}}_E $, $ \tilde t = t - E $ holds.
Then, the update rule of local model can be written as 
\begin{equation}
	{\mathbf{w}_{t,j}} 		= \left\{ \begin{array}{l}
		{\mathbf{w}_{\tilde t,j}} - \sum\limits_{s = \tilde t}^{t - 1} {{\eta _s}{\mathbf{g}_{s,j}}},  \;  {\text{if}} \; t \notin {\mathcal{I}}_E, \\ 
		\frac{{1 - \alpha {A_t}}}{{{N_i}}}\sum\limits_{j \in {{{\cal N}}_i}} {\left( {{\mathbf{w}_{\tilde t,j}} - \sum\limits_{s = \tilde t}^{t - 1} {{\eta _s}{\mathbf{g}_{s,j}}} } \right)}  + \alpha {A_t}{{\overline {\mathbf{w}}}_{t - \delta E ,i - 1}},  \\
		\qquad \qquad \qquad \qquad \qquad \qquad 	\qquad \;  {\text{if}} \; t \in {\mathcal{I}}_E, \\
	\end{array} \right.
\end{equation}
where $ {\mathbf{g}_{s,j}} = {\nabla {F_j}\left( {\mathbf{w}_{s,j},\varsigma _{s,j}} \right)} $ for notational brevity.
The virtual global model can be written as
\begin{equation}
	\begin{split}
		&	{{\overline {\mathbf{w}}}_t} = \left( {1 - \alpha {A_t}} \right){{\overline {\mathbf{v}}}_t} + \alpha {A_t}{{\overline {\mathbf{w}}}_{t - \delta E }} \\
		& = \frac{1}{M}\sum\limits_{i \in {\mathcal{M}}} {\left[ {\left( {1 - \alpha {A_t}} \right)\left( {{{\overline {\mathbf{w}}}_{\tilde t,i}} - \sum\limits_{s = \tilde t}^{t - 1} {{\eta _s}{\mathbf{g}_{s,i}}} } \right) + \alpha {A_t}{{\overline {\mathbf{w}}}_{t - \delta E ,i - 1}}} \right]}.
	\end{split}
\end{equation}

First, we discuss the upper bound of $ Q_t $ when $ t \in {\mathcal{I}}_E$, and we have
\begin{align}
	&	\mathbb{E}{Q_t} = \sum\limits_{i \in {\mathcal{M}}} {\sum\limits_{j \in {{\mathcal{N}}_i}} {\frac{1}{{M{N_i}}}\mathbb{E}{{\left\| {{\mathbf{w}_{t,j}} - {{\overline {\mathbf{w}}}_t}} \right\|}^2}} }  \nonumber \\
	&=	\sum\limits_{i \in {\mathcal{M}}} {\sum\limits_{j \in {{\mathcal{N}}_i}} {\frac{1}{{M{N_i}}}\mathbb{E}} } \left\| {\left( {1 - \alpha {A_t}} \right)\left( {{\mathbf{w}_{\tilde t,j}} - {{\overline {\mathbf{w}}}_{\tilde t}}} \right)} \right. \nonumber \\
	&{ + \alpha {A_t}\left( {{{\overline {\mathbf{w}}}_{t - \delta E ,i - 1}} - \frac{1}{M}\sum\limits_{i \in {\mathcal{M}}} {{{\overline {\mathbf{w}}}_{t - \delta E ,i - 1}}} } \right)}  \nonumber \\
	& \left. { - \left( {1 - \alpha {A_t}} \right)\left( {\sum\limits_{s = \tilde t}^{t - 1} {{\eta _s}{\mathbf{g}_{s,j}}}  - \frac{1}{M}\sum\limits_{i \in {\mathcal{M}}} {\sum\limits_{s = \tilde t}^{t - 1} {{\eta _s}{\mathbf{g}_{s,i}}} } } \right)} \right\| ^2 \nonumber \\
	&\leq \left(1+b\right)	\sum\limits_{i \in {\mathcal{M}}} {\sum\limits_{j \in {{\mathcal{N}}_i}} {\frac{1}{{M{N_i}}}\mathbb{E}} } \left\| {\left( {1 - \alpha {A_t}} \right)\left( {{\mathbf{w}_{\tilde t,j}} - {{\overline {\mathbf{w}}}_{\tilde t}}} \right)} \right. \nonumber \\
	&\left. { + \alpha {A_t}\left( {{{\overline {\mathbf{w}}}_{t - \delta E ,i - 1}} - {\overline {\mathbf{w}}}_{t - \delta E }} \right)} \right\| ^2  \nonumber \\
	& + \left( {1 + \frac{1}{b}} \right)\sum\limits_{i \in {\mathcal{M}}} {\sum\limits_{j \in {\mathcal{N}_i}} {\frac{1}{{M{N_i}}}\mathbb{E}\left\| {\left( {1 - \alpha {A_t}} \right)} \right.} } \nonumber \\
	\label{equ:jensen_inequality}	& {\left. { \cdot \left( {\sum\limits_{s = \tilde t}^{t - 1} {{\eta _s}{\mathbf{g}_{s,j}}}  - \frac{1}{M}\sum\limits_{i \in {\mathcal{M}}} {\sum\limits_{s = \tilde t}^{t - 1} {{\eta _s}{\mathbf{g}_{s,i}}} } } \right)} \right\|^2}  \\
	& \leq \left( {1 + b} \right){\left( {1 - \alpha {A_t}} \right)^2}\sum\limits_{i \in {\mathcal{M}}} {\sum\limits_{j \in {{\mathcal{N}}_i}} {\frac{1}{{M{N_i}}}\mathbb{E}{{\left\| {{\mathbf{w}_{\tilde t,j}} - {{\overline {\mathbf{w}}}_{\tilde t}}} \right\|}^2}} } \nonumber\\
	&  + \left( {1 + b} \right){\left( {\alpha {A_t}} \right)^2}\sum\limits_{i \in {\mathcal{M}}} {\sum\limits_{j \in {{\mathcal{N}}_i}} {\frac{1}{{M{N_i}}}\mathbb{E}{{\left\| {{{\overline {\mathbf{w}}}_{t - \delta E ,i - 1}} - {{\overline {\mathbf{w}}}_{t - \delta E }}} \right\|}^2}} } \nonumber \\
	& + \left( {1 + \frac{1}{b}} \right){\left( {1 - \alpha {A_t}} \right)^2}\sum\limits_{i \in {\mathcal{M}}} \sum\limits_{j \in {{\mathcal{N}}_i}} {\frac{1}{{M{N_i}}} } \nonumber \\
	\label{equ:convexity_squarenorm}	& \cdot {\mathbb{E}{{\left\| {\sum\limits_{s = \tilde t}^{t - 1} {{\eta _s}{\mathbf{g}_{s,j}}}  - \frac{1}{M}\sum\limits_{i \in {\mathcal{M}}} {\sum\limits_{s = \tilde t}^{t - 1} {{\eta _s}{\mathbf{g}_{s,i}}} } } \right\|}^2}}  \\ 
	& \leq \left( {1 + b} \right){\left( {1 - \alpha {A_t}} \right)^2}{Q_{\tilde t}} + \left( {1 + b} \right){\left(\alpha {A_t}\right) ^2}{Q_{t - \delta E }} \nonumber \\
	& + 4\left( {1 + \frac{1}{b}} \right){\left( {1 - \alpha {A_t}} \right)^2}{\left(E-1\right)^2}{G^2}\eta _t^2, \nonumber
\end{align}
where we use Jensen inequality in (\ref{equ:jensen_inequality}). That is, for any $ X_1, X_1 \in \mathbb{R}^d $, $ \|X_1 + X_2\|^2 \leq (1 + b) \|X_1\|^2 + \left(1 + \frac{1}{b}\right) \|X_2\|^2 $, where $ b $ can be any positive real number.
In (\ref{equ:convexity_squarenorm}), the first two terms come from the convexity of $\left\| \cdot \right\|^2  $, and the expectation in the last term can be calculated by
\begin{align*}
	\mathbb{E}{\left\| {\sum\limits_{s = \tilde t}^{t - 1} {{\eta _s}{\mathbf{g}_{s,j}}}  - \frac{1}{M}\sum\limits_{i \in {\mathcal{M}}} {\sum\limits_{s = \tilde t}^{t - 1} {{\eta _s}{\mathbf{g}_{s,i}}} } } \right\|^2} 	\leq \mathbb{E}{\left\| {\sum\limits_{s = \tilde t}^{t - 1} {{\eta _s}{\mathbf{g}_{s,j}}} } \right\|^2} \\
	\leq \mathbb{E}{\left\| {\sum\limits_{s = \tilde t}^{t - 1} {{\eta _{\tilde t}}{\mathbf{g}_{s,j}}} } \right\|^2} 
	\leq 4\eta _t^2{\left(E-1\right)^2}{G^2},
\end{align*}
where the first inequality comes from $ \mathbb{E} \| X - \mathbb{E}X \|^2 \leq \mathbb{E} \| X \|^2 $.

When $ t \notin {\mathcal{I}}_E$, we have
\begin{equation}
	\begin{split}
		&	\mathbb{E}{Q_t} = \sum\limits_{i \in {\mathcal{M}}} {\sum\limits_{j \in {{\mathcal{N}}_i}} {\frac{1}{{M{N_i}}}\mathbb{E}{{\left\| {{\mathbf{w}_{t,j}} - {{\overline {\mathbf{w}}}_t}} \right\|}^2}} }  \\
		&	= \sum\limits_{i \in {\mathcal{M}}} {\sum\limits_{j \in {{\mathcal{N}}_i}} {\frac{1}{{M{N_i}}}\mathbb{E}\left\| {\left( {{\mathbf{w}_{\tilde t,j}} - {{\overline {\mathbf{w}}}_{\tilde t}}} \right)} \right.} }   \\
		&	{\left. { - \left( {\sum\limits_{s = \tilde t}^{t - 1} {{\eta _s}{\mathbf{g}_{s,j}}}  - \frac{1}{M}\sum\limits_{i \in {\mathcal{M}}} {\sum\limits_{s = \tilde t}^{t - 1} {{\eta _s}{\mathbf{g}_{s,i}}} } } \right)} \right\|^2} \\
		& \leq \left( {1 + b} \right)\sum\limits_{i \in {\mathcal{M}}} {\sum\limits_{j \in {{\mathcal{N}}_i}} {\frac{1}{{M{N_i}}}\mathbb{E}{{\left\| {{\mathbf{w}_{\tilde t,j}} - {{\overline {\mathbf{w}}}_{\tilde t}}} \right\|}^2}} } \\
		& + \left( {1 + \frac{1}{b}} \right)\sum\limits_{i \in {\mathcal{M}}} \sum\limits_{j \in {{\mathcal{N}}_i}} {\frac{1}{{M{N_i}}} } \nonumber \\
		& \cdot {\mathbb{E}{{\left\| {\sum\limits_{s = \tilde t}^{t - 1} {{\eta _s}{\mathbf{g}_{s,j}}}  - \frac{1}{M}\sum\limits_{i \in {\mathcal{M}}} {\sum\limits_{s = \tilde t}^{t - 1} {{\eta _s}{\mathbf{g}_{s,i}}} } } \right\|}^2}}  \\ 
		&  \leq \left( {1 + b} \right){Q_{\tilde t}} + 4\left( {1 + \frac{1}{b}} \right){E^2}{G^2}\eta _t^2.
	\end{split}		
\end{equation}

This completes the proof.

\section{Proof of Theorem \ref{theorem:Q_T_UpperBound}} \label{proof:theorem_Q_T_UpperBound}

First, we define the following variables $ C_1 - C_5 $ for notational brevity. That is, $ {C_1} = \left( {1 + b} \right){\left( {1 - \alpha } \right)^2} $, $ {C_2} = \left( {1 + b} \right){\alpha ^2} $, $ {C_3} = 4\left( {1 + \frac{1}{b}} \right){\left( {1 - \alpha } \right)^2}{\left( {E - 1} \right)^2}{G^2} $, $ {C_4} = \left( {1 + b} \right) $, and $ {C_5} = 4\left( {1 + \frac{1}{b}} \right){\left( {E - 1} \right)^2}{G^2} $.
When $ R- 1 $ is the step of different area mixing models, i.e., $ A_{R-1} = 1 $, we have the following set of inequalities from step $ {R - 1 - \left( {K + \delta } \right)E + E} $ to $ R-1 $ based on (\ref{equ:relationship_Qt}).
\begin{equation*}
  \left\{ \begin{gathered}
	{Q_{R - 1}} \leq {C_1}{Q_{R - 1 - E}} + {C_2}{Q_{R - 1 - \delta E}} + {C_3}\eta _{R - 1}^2, \hfill \\
	{Q_{R - 1 - E}} \leq {C_4}{Q_{R - 1 - 2E}} + {C_5}\eta _{R - 1 - E}^2, \hfill \\
	\ldots  \hfill \\
	{Q_{R - 1 - \left( {K + \delta } \right)E + E}} \leq {C_4}{Q_{R - 1 - \left( {K + \delta } \right)E}} + {C_5}\eta _{R - 1 - \left( {K + \delta } \right)E + E}^2. \hfill \\ 
\end{gathered}  \right. 
\end{equation*}
Based on the recurrence relation, we have
\begin{equation}\label{equ:QT_kappa}
	{Q_{R - 1}} \leq {\kappa _1}{Q_{R - 1 - \left( {K + \delta } \right)E}} + {\kappa _2}\eta _{R - 1 - \left( {K + \delta } \right)E + E}^2,
\end{equation}
where $ {\kappa _1} = {C_1}{\left( {{C_4}} \right)^{K + \delta  - 1}} + {C_2}{\left( {{C_4}} \right)^R} ={\left( {1 + b} \right)^{K + \delta }}{\left( {1 - \alpha } \right)^2} + {\left( {1 + b} \right)^{K + 1}}{\alpha ^2} $ and $ {\kappa _2} = {C_1}{C_5} \allowbreak \sum\limits_{l = 1}^{K + \delta  - 1} {{{\left( {{C_4}} \right)}^{l - 1}}}  + {C_2}{C_5}\sum\limits_{l = \delta }^{K + \delta  - 1} {{{\left( {{C_4}} \right)}^{l - \delta }}}  + {C_3} = 4\left( {1 + \frac{1}{b}} \right){\left( {E - 1} \right)^2}{G^2}\left[ {\frac{{{{\left( {1 + b} \right)}^{K + \delta }}{{\left( {1 - \alpha } \right)}^2} + {{\left( {1 + b} \right)}^{K + 1}}{\alpha ^2}}}{b} + {{\left( {1 - \alpha } \right)}^2}} \right] $.

Assume that  $ R -1= q \left( K + \delta\right)E  $, where $ q $ denotes the total model mixing times from $ t=0 $ to $ R-1 $. With the similar form of (\ref{equ:QT_kappa}), we have the following inequalities of $ Q_t $ when $ A_t = 1 $.
\begin{equation*}
   \resizebox{1\hsize}{!}{$ \left\{ \begin{gathered}
	{Q_{R - 1}} \leq {\kappa _1}{Q_{R - 1 - \left( {K + \delta } \right)E}} + {\kappa _2}\eta _{T -1- \left( {K + \delta } \right)E+E}^2, \hfill \\
	{Q_{R - 1 - \left( {K + \delta } \right)E}} \leq {\kappa _1}{Q_{R - 1 - 2\left( {K + \delta } \right)E}} + {\kappa _2}\eta _{R - 1 - 2\left( {K + \delta } \right)E + E}^2, \hfill \\
	... \hfill \\
	{Q_{R - 1 - \left( {q - 1} \right)\left( {K + \delta } \right)E}} \leq {\kappa _1}{Q_{R - 1 - q\left( {K + \delta } \right)E}} + {\kappa _2}\eta _{R - 1 - q\left( {K + \delta } \right)E + E}^2, \hfill \\ 
\end{gathered}  \right. $}
\end{equation*}
where the last inequality can also be represented by $ {Q_E} \leq {\kappa _1}{Q_0} + {\kappa _2}\eta _E^2 $.
According to the recurrence relationship, $ Q_{R-1} $ can be upper bounded by
\begin{align*}
	{Q_{R - 1}} & \leq {\kappa_1^q}{Q_0} + \kappa_2\sum\limits_{l = 0}^{q - 1} {{\kappa_1^l}\eta _{R - 1 - \left( l+1\right) \left( {K + \delta } \right) +E}^2} \\
	& = \kappa_2\sum\limits_{l = 0}^{q - 1} {{\kappa_1^l}\eta _{R - 1 - \left( l+1\right) \left( {K + \delta } \right) +E}^2} \\
	& \leq \eta _E^2\kappa_2\sum\limits_{l = 0}^{q - 1} \kappa_1^l   \leq \eta _E^2\frac{{{\kappa _2}}}{{1 - {\kappa _1}}},
\end{align*}
where the first equality follows from $ Q_0 = 0  $ due to initialization. This completes the proof.

\section{Proof of Theorem \ref{theorem:convergence_bound_partial}}\label{proof:convergence_bound_partial}

According to Lemma \ref{lemma:unbiased_sampling}, $ {\mathbb{E}_{{{\mathcal{U}}_t}}}{\left\| {{{\overline {\mathbf{z}}}_t} - {\mathbf{w}^ * }} \right\|^2} $ has the following upper bound,
\begin{align}
	&	{\mathbb{E}_{{{\mathcal{U}}_t}}}{\left\| {{{\overline {\mathbf{z}}}_t} - {\mathbf{w}^ * }} \right\|^2} = {\mathbb{E}_{{{\mathcal{U}}_t}}}{\left\| {{{\overline {\mathbf{z}}}_t} - {{\overline {\mathbf{w}}}_t} + {{\overline {\mathbf{w}}}_t} - {\mathbf{w}^ * }} \right\|^2} \nonumber \\
	\label{equ:vanish_term}	&	= {\mathbb{E}_{{{\mathcal{U}}_t}}}{\left\| {{{\overline {\mathbf{z}}}_t} - {{\overline {\mathbf{w}}}_t}} \right\|^2} + {\left\| {{{\overline {\mathbf{w}}}_t} - {\mathbf{w}^ * }} \right\|^2} + 2{\mathbb{E}_{{{\mathcal{U}}_t}}}\left\langle {{{\overline {\mathbf{z}}}_t} - {{\overline {\mathbf{w}}}_t},{{\overline {\mathbf{w}}}_t} - {\mathbf{w}^ * }} \right\rangle  \\
	&	\leq \frac{1}{{{M^2}}}\sum\limits_{i \in {\mathcal{M}}} {\frac{{{N_i} - {U_i}}}{{{U_i}{N_i}({N_i} - 1)}}} \sum\limits_{j \in {{\mathcal{N}}_i}} {{{\left\| {{\mathbf{w}_{t,j}} - {{\overline {\mathbf{w}}}_t}} \right\|}^2}}  + {\left\| {{{\overline {\mathbf{w}}}_t} - {\mathbf{w}^ * }} \right\|^2} \nonumber \\
	&	\leq \mathop {\max }\limits_{i \in \mathcal{M}} \left\{ {\frac{{{N_i} - {U_i}}}{{M{U_i}\left( {{N_i} - 1} \right)}}} \right\}\sum\limits_{i \in {\mathcal{M}}} {\sum\limits_{j \in {{\mathcal{N}}_i}} {\frac{1}{{M{N_i}}}{{\left\| {{\mathbf{w}_{t,j}} - {{\overline {\mathbf{w}}}_t}} \right\|}^2}} } \nonumber \\
	& + {\left\| {{{\overline {\mathbf{w}}}_t} - {\mathbf{w}^ * }} \right\|^2} \nonumber \\
	&	= \mathop {\max }\limits_{i \in \mathcal{M}} \left\{ {\frac{{{N_i} - {U_i}}}{{M{U_i}\left( {{N_i} - 1} \right)}}} \right\}{Q_t} + {\Delta _t}, \nonumber 
\end{align}
where the last term in (\ref{equ:vanish_term}) is eliminated in the next line due to the unbiasedness of $ {\overline {\mathbf{z}}}_t $. Similar to (\ref{equ:gap_barw_wstar}), we have
\begin{align*}
	&	\mathbb{E}\left[ {F\left( {{{\overline {\mathbf{z}}}_t}} \right)} \right] - F\left( {{{\mathbf{w}}^{ \star }}} \right) \leq \frac{L}{2}{{\mathbb{E}_{{{\mathcal{U}}_t}}}{\left\| {{{\overline {\mathbf{z}}}_t} - {\mathbf{w}^ * }} \right\|^2}} \\
	& \le \frac{L}{2}\left[ {\mathop {\max }\limits_{i \in {{\cal M}}} \left\{ {\frac{{{N_i} - {U_i}}}{{M{U_i}\left( {{N_i} - 1} \right)}}} \right\}{Q_t}} \right.\\
	&	\left. { + \left( {\frac{1}{{1 - \alpha }} - \alpha  + \frac{1}{2}} \right)\frac{v}{{t + \gamma }} + 2{Q_{t - 1}}} \right].
\end{align*}
When $ t=R $ and $ A_{R-1} = 1 $, $ Q_R $ can be obtained by (\ref{equ:Q_t_noagg}) where $ \tilde{t} = R-1 $. This completes the proof.

\end{appendices}

\bibliographystyle{IEEEtran}
\bibliography{IEEEabrv,reference}

\end{document}